\documentclass[11pt]{article}
\usepackage[margin=1in]{geometry}

\usepackage{hyperref}
\usepackage{url}

\usepackage{graphicx} 

\usepackage{fullpage}

\usepackage{amsmath}
\usepackage{amssymb}
\usepackage{mathtools}
\usepackage{amsthm}
\usepackage{balance}
\usepackage{comment}
\usepackage{algorithm}
\usepackage{algorithmicx}
\theoremstyle{definition}
\usepackage{subcaption}

\usepackage{booktabs} 
\usepackage{subcaption}
\usepackage{multirow}
\usepackage{comment}

\newtheorem{theorem}{Theorem}[section]
\newtheorem{lemma}{Lemma}[section]

\newtheorem{corollary}{Corollary}[section]
\newtheorem{definition}{Definition}[section]

\newtheorem{proposition}{Proposition}[section]
\newtheorem{remark}{Remark}[section]

\usepackage{algpseudocode}

\usepackage{todonotes}
\newcommand{\mytodo}[1]{\ifnum\Comments=1{#1}\fi}

\newcommand{\ignore}[1]{{}}

\usepackage[capitalize,noabbrev,nameinlink]{cleveref}

\newcommand{\R}{\mathbb{R}}
\newcommand{\E}{\mathbb{E}}

\newcommand{\matX}{\mathbf{X}}

\title{Simple Mechanisms for Representing, Indexing and Manipulating Concepts}

\author{
\begin{minipage}[t]{0.33\linewidth}
    \centering
    Yuanzhi Li \\ 
    Carnegie Mellon University \\
    \texttt{yuanzhil@andrew.cmu.edu}
  \end{minipage}%
  \begin{minipage}[t]{0.33\linewidth}
    \centering
    Raghu Meka \\ 
    Google Research \\
    \texttt{raghum@cs.ucla.edu}
  \end{minipage} \vspace{5mm} \\ 
  \begin{minipage}[t]{0.33\linewidth}
    \centering
    Rina Panigrahy \\ 
    Google Research \\
    \texttt{rinap@google.com}
  \end{minipage}%
\begin{minipage}[t]{0.33\linewidth}
    \centering
    Kulin Shah \thanks{Work done during an internship at Google Research. } \\ 
    UT Austin \\
    \texttt{kulinshah@utexas.edu}
  \end{minipage}%
}

\definecolor{myred}{RGB}{219, 50, 54}

\hypersetup{
    colorlinks=true,
	citecolor=blue,
	linkcolor=myred
	}

\begin{document}

\maketitle

\begin{abstract}
Supervised and unsupervised learning using deep neural networks typically aims to exploit the underlying structure in the training data; this structure is often explained using a latent generative process that produces the data, and the generative process is often hierarchical, involving latent concepts. Despite the significant work on understanding the learning of the latent structure and underlying concepts using theory and experiments, a framework that mathematically captures the definition of a concept and provides ways to operate on concepts is missing. In this work, we propose to characterize a simple primitive concept by the zero set of a collection of polynomials and use moment statistics of the data to uniquely represent the concepts; we show how this view can be used to obtain a signature of the concept. These signatures can be used to discover a common structure across the set of concepts and could recursively produce the signature of higher-level concepts from the signatures of lower-level concepts. To utilize such desired properties, we propose a method by keeping a dictionary of concepts and show that the proposed method can learn different types of hierarchical structures of the data. 
\end{abstract}

\section{Introduction}
Learning methods implicitly try to exploit the underlying structure of the data to make predictions using either supervised or unsupervised training  \cite{bengio2009learning,Goodfellow-et-al-2016}. The underlying structure is often believed to have arisen based on some latent generative process that must have produced the data, which in turn may involve latent concepts. For example, latent concepts in an image may be objects such as dogs, human faces, etc; in an abstract sense, these concepts may typically be hierarchical with advanced concepts built on top of the simpler concepts such as eyes, nose or even simple geometric shapes such as lines, circles, and rectangles. 
It has been shown that when deep networks are trained on the labeled data, they learn the hierarchical concept structure that is useful to solve a particular task \cite{allen2020backward,chen2020towards,zeiler2014visualizing,qi2017pointnet++}. A similar phenomenon has also been observed when the networks are trained using unsupervised or self-supervised training \cite{chen2020simple,jing2020self,Manning2020EmergentLS}. Even though a long line of work on understanding the learning of concept structure using experiments \cite{zeiler2014visualizing,qi2017pointnet++,Manning2020EmergentLS} and theory \cite{allen2020backward,chen2022towards}, a framework that mathematically characterizes what a concept is, and how they can be represented and discovered, is missing. Additionally, such methods are empirical and without any theoretical guarantees. Therefore, developing tools and frameworks to better contextualize, understand, and manipulate learned concepts is an important area of research. Especially as models become more and more complex, having the right types of abstractions will become critical for better understanding and developing new architectures. In this work, we take a step towards such a formal understanding of the notion of {\em concepts}. 

At the simplest level, a concept may be modeled using a manifold. For example, the concept of any shape (e.g., the concept of a circle) in high dimensions lies in a low-dimensional manifold. While there has been a long line of work on learning manifolds from the data distribution~\cite{izenman2012introduction,caterini2021rectangular,bashiri2018multi,pedronette2018unsupervised,wang2018flexible,lin2008riemannian,brehmer2020flows,brosch2013manifold,han2022enhance,cayton2008algorithms,ma2011manifold,zhu2018image,lunga2013manifold,talwalkar2008large}, this line of work fails to answer the question of learning a hierarchical structure of concepts of the data.

In this work, we start with primitive concepts (such as simple geometric shapes), defined by a manifold that can be specified using the zero sets of polynomial equations (for example a circle is defined by its equation; such manifolds are also called algebraic manifold in the literature\footnote{We want to point out that the precise mathematical definition of the algebraic manifold in their full generality is quite technical. Therefore, we provide a definition that will suffice for this work and we call it a well-behaved algebraic manifold. Even though for brevity we write algebraic manifold at some places, note that we always mean the well-behaved algebraic manifold.}). We show that the null space of the moments of the distribution on such a manifold can identify the manifold. In other words, these elementary statistics can serve as a \emph{signature} of the concept. Importantly, we show how these ideas can be generalized to higher level concepts that build upon primitive concepts; further we show how the idea of using null spaces can recursively produce the signature of the higher-level concepts from the signatures of the lower-level concepts, and naturally gives us an algorithm to discover structure across the set of concepts.

Although our definition of a concept is entirely mathematical, we show how this mathematical definition at an intuitive level maps the process of finding a concept to the transformer architecture~\cite{vaswani2017attention};  we will argue that the attention mechanism naturally groups inputs corresponding the same concept and the MLP layers can compute the moment statistics and the corresponding null space leading to the signature of the manifolds, thus matching our proposed definition. This view of how concept signatures are discovered using a transformer architecture also suggests that it would be advantageous to augment  each layer with a concept memory table that holds concept signatures where tables at higher layers hold concepts that build upon simple concepts in lower layers, thus separating concept disccovery from concept storage.

While there have been many interesting experimental works showing that the feedforward networks encode concept information \cite{geva-etal-2022-transformer}, visualizing the feed-forward network as a key-value memory unit \cite{geva-etal-2021-transformer} or keeping the explicit memory component to store the knowledge \cite{lewis2020retrieval,wu2022memorizing}, they fail to explain how the concepts are \emph{discovered} and how the lower level concepts are combined to obtain the higher-level concepts. 

\subsection{Our contribution}

Our main results can be summarized as follows:

 \begin{itemize}
 
     \item \textbf{Null space signature and membership inference for a concept}. We provide a simple method to compute concept signatures when the concept is given by a low-dimensional manifold. The signature is obtained by computing the null space of the moment statistics. Additionally, one can check if a point lies on the manifold by just checking the inner product between the signature and a non-linear transformation of the point (\cref{thm:nullspacepolyeqns}). For constant degree manifolds of constant dimensions the signatures are of constant size independent of the ambient dimension (\cref{thm:jl-manifold-lemma}). And similar concepts will have similar signatures (\cref{lemma:similar-manifolds})

     \item \textbf{Higher level concepts from lower level concepts}. We show that the null space signature can be used to obtain signature of higher (more abstract) level concepts from lower level concepts. For example, the signature of individual concepts can be used to obtain the signature of the intersection of two concepts  (\cref{thm:structuresigs}). In \cref{thm:dictionarysigs}, we show that if a collection of higher level concepts can be obtained from sparse unions of atomic concepts then the signatures of the atomic concepts can be obtained from the signatures of the higher level concepts. We provide additional such examples in \cref{appendix-new-sec:examples}.

     \item \textbf{Connections to transformer architecture}. 
     There are simple instantiations of a transformer network with mostly \textit{random} sets of parameters and a few learned projections that can compute concept signatures of latent manifolds present in the input data (see \cref{sec:transformer}). Our view also motivates an enhanced architecture that attaches a table of concept signature at each layer where concepts at higher layers build upon simpler concepts in lower layers -- this type of architecture separates concept {\em discovery} (done by attention, MLP layers) from concept {\em storage} (stored in the concept table).

     \item \textbf{Experiments.} We validate our hypothesis on learning higher level concepts by lower level concepts using the null space signature by experiments on synthetic data in \cref{appendix:experiments}. 
 \end{itemize}

Our analysis points towards an architecture that is very similar to the transformer but in addition, also has a dictionary of concept signatures at each layer that is produced over time in a system that is receiving a continuous stream of inputs. The set of signatures at each layer is looked up as new inputs arrive for the inference. This architecture is similar to the architectures for retrieval-based methods that keep an explicit memory unit to store the knowledge \cite{lewis2020retrieval,wu2022memorizing}. We consider that our work provides a theoretical justification that performing simple operations on the stored knowledge can improve the performance.

\section{Subspace signatures of the concepts}
\label{sec:concent-signature-intro}

As mentioned earlier, through out the paper, we will study manifolds, mostly well-behaved algebraic manifolds (defined in \cref{def:algebraic-manifold}), but will also have a few results for more general \emph{anlaytic manifolds}. For brevity, we will shorten well-behaved algebraic manifolds to \emph{algebraic manifolds} throughout (keeping in mind that our notion is more restricted than general manifolds). 

\paragraph{Notations.} We will use the following notational conventions: For two tensors $A, B$ of the same dimension, we use $\langle A, B \rangle \in \mathbb{R}$ to denote the dot product between $A$ and $B$. For any vector $x$, $x^{\otimes l}$ denotes the $l$'th tensor power of $x$ and $[1; x]$ denotes the concatenation of $1$ to $x$ vector. Let $\phi(x) : \mathbb{R}^d \to \mathbb{R}^m$ denote the feature mapping. Throughout this work, we will use polynomial features  $[1;x]^{\otimes l}$.

In this section, we define our main notion of a manifold signature. We start with a linear manifold in $d$-dimensional space.

\subsection{Warm up: a linear manifold in \texorpdfstring{$d$}{d}-dimensional space}

We start with the example of the points coming from a linear subspace of dimension $k$. For simplicity, we will focus on linear manifolds that pass through the origin. 

\begin{definition}[Linear $k$-dimensional manifold] 
\label{def:linear-manifold}
    Let $w_1, w_2, \ldots, w_{d-k}$ be linearly independent vectors in $d$-dimensional space that are orthogonal to the manifold. We say that a data distribution $\mathcal D$ lies in corresponding linear $k$-dimensional manifold $\mathcal M$ if for all $x$ in support of the distribution $\mathcal D$, $w_i^\top x = 0$ for all $i \in [d-k]$ \footnote{We also assume that $w_1, w_2, \ldots, w_{d-k}$ is a maximal linearly independent vector set for which $w_i^\top x = 0$ and there does not exist any other linearly independent vector $w_{d-k+1}$ such that $w_{d-k+1}^\top x = 0$ for all points $x$ in support of distribution $\mathcal D$. This condition intuitively ensures that distribution $\mathcal D$ spans the $k$-dimensional manifold and does not lie in a smaller than $k$-dimensional manifold.}. 
\end{definition}

Note that the assumption of independence of $w_1, w_2, \ldots, w_{d-k}$ vectors is without loss of generalization because if the set of vectors is dependent then one can covert into the set of independent vectors by Principal Component Analysis (PCA). To understand the definition, one can consider the distribution on a hyperplane, i.e., for all $x$ in the support of the distribution $\mathcal D$, $w^\top x = 0$.  As the support of such a distribution can be at most $d-1$ dimensional, we call it a $(d-1)-$dimensional manifold.

\begin{definition}[Null space signature for a linear manifold]
\label{def:signature-linear-manifold}
Let $X$ be a random variable corresponding to a distribution, $\mathcal D$, that lies in a linear $k$-dimensional manifold. We denote the second moment of $X$ by $M(X)$ (i.e., $M(X) = \E_{X \sim \mathcal D}[ X X^\top ]$). Let $U \Sigma U^\top$ denote the eigendecomposition of $M(X)$ where $\Sigma \in \R^{d \times d}$ is a diagonal matrix with the decreasing order of eigenvalues ($\Sigma_{1, 1} \geq \Sigma_{2, 2} \ldots \geq \Sigma_{k, k}$ and $\Sigma_{i, i} = 0$ for all $i \in \{ k+1, \ldots, d \}$). We define a null-space signature of the $k$-dimensional linear manifold as $T(X) = U_{k+1:d} U_{k+1:d}^{\top}$ where $U_{(k+1):d} \in \R^{d \times (d-k)}$ use $(k+1)^{th}$ to $d^{th}$ eigenvectors (eigenvectors corresponding to zero eigenvalues).
\end{definition}

Note that the expectation in \cref{def:signature-linear-manifold} can be replaced by an empirical estimate of sufficiently many samples. Through out the paper, we will present all our results with the expectation for simplicity however, they can be generalized to the setting when we only have empirical estimates and are included in the appendix.

We first show that the null space signature defined in \cref{def:signature-linear-manifold} can be used to identify points on the manifold.

\begin{proposition}
\label{prop:linear-manifold}
    The null space signature $T(X)$ of a $k$-dimensional linear manifold $\mathcal M$ (defined in \cref{def:linear-manifold}) uniquely identifies the manifold of the distribution $\mathcal D$. That is, for any point $x$ on the manifold $\mathcal M$, $\langle T(X), x x^\top \rangle = 0$ and for any point $x$ not on the manifold $\mathcal M$, $\langle T(X), x x^\top \rangle > 0$ and equals the distance of $x$ from the manifold.
    
\end{proposition}

The proof of \cref{prop:linear-manifold} is provided in \cref{appendix-new-subsec:linear-manifold-proofs}.

\textit{Proof idea.} To explain the main idea behind the proof of the above proposition, we start with an example of a $d-1$-dimensional manifold $\mathcal M$. That is, points $x$ are coming from a distribution $\mathcal D$ on $\mathcal M$ such that for all $x$ in the support of $\mathcal D$, $w_1^\top x = 0$. In other words, this means that the manifold $\mathcal M$ is orthogonal to the vector $w_1$. Note that column space of $M(X)$ captures the subspace of distribution $\mathcal D$ (because $M(X) = \int p(x) xx^\top dx$ is an integral over psd matrices $xx^\top$) and therefore, $w_1$ will be the only direction along which $M(X)$ will have zero eigenvalue because of zero eigenvalue corresponding to $w_1$ for all $x x^\top$ for any $x \in \text{supp}(x)$. Therefore, in this case the null space signature will become $T(X) = w_1 w_1^\top$. Additionally, for any point $x$ on the manifold $\mathcal M$, $\langle T(X), x x^\top \rangle = (w_1^\top x)^2 = 0$ and for any point $x$ not on $\mathcal M$ will have non-zero component in the direction of $w_1$ therefore, $\langle T(X), x x^\top \rangle = (w_1^\top x)^2 > 0$.

To generalize the above idea for a $k$-dimensional manifold, we observe that $M(X)$ will have zero eigenvalues not only in the direction of $w_1$ but for $d-k$ eigenvectors in the subspace formed by $w_1, w_2, \ldots, w_{d-k}$ vectors because column space of $x x^\top$ lies in the orthogonal of the subspace of $w_1, w_2, \ldots, w_{d-k}$. Therefore, $T(X)$ will capture the subspace of $w_1, \ldots, w_{d-k}$ and similarly to the $(d-1)$-dimensional manifold case, the membership check will follow.

On an intuitive level, the null space signature $T(X)$ captures the subspace of all orthogonal directions to the manifold $\mathcal M$. Additionally, it \textit{does not depend on the distribution $\mathcal D$ on the manifold $\mathcal M$.}

\subsection{Extending to well-behaved algebraic manifold}
\label{subsec:well-behaved-manifold}

In this section, we extend our main notion of a \emph{signature} for a set of points that is meant to capture the notion of concepts. We then show certain properties like membership inference of the signature when the underlying points lie on a \emph{well-behaved algebraic manifold} (see definition below). 

\begin{definition}[Well-behaved Algebraic Manifold]
\label{def:algebraic-manifold}
Let $\phi(x) : \R^d \to \R^{m}$ be a feature mapping from $x$ to all monomials of degree $\ell$ of $x$ (hence, $m = O(d^{\ell})$). Let $w_1, w_2, \ldots, w_{d-k}$ be linearly independent vectors in $m$-dimensional space orthogonal to the manifold. We say a distribution lies in $k-$dimensional well-behaved manifold $\mathcal M$ if for all $x$ in support of $\mathcal D$, $w_i^\top \phi(x) = 0$ for all $i \in [d-k]$. Additionally, the degree of the feature mapping $\ell$ is called the degree of manifold $\mathcal M$.

\end{definition}

Throughout the paper, the coefficient vectors $w_i$ that represent the polynomials $P_i(x) = w_i^\top \phi(x)$ have been orthonormalized. Let $P(x)$ denote the vector $(P_1(x), P_2(x), \ldots, P_{d-k}(x))$.

Note that the only difference in the definition of a well-behaved algebraic manifold from the linear manifold in \cref{def:linear-manifold} is that the former one is defined on the polynomial feature mapping $\phi(x)$. Similar to \cref{def:signature-linear-manifold}, we now define the signature for the algebraic manifold.

\begin{definition}\label{def:signature-polynomial-manifold}
    Let $X$ be a random variable corresponding to distribution $\mathcal D$ that lies in a $k$-dimensional well-behaved algebraic manifold. Let $M(X)$ be the second moment of the feature mapping $\phi(x)$ with eigendecomposition $M(X) = \E_{x \sim \mathcal D}[ \phi(x) \phi(x)^\top ] = U \Sigma U^\top$ where $\Sigma \in R^{m \times m}$ is a diagonal matrix with decreasing order of eigenvalues $(\Sigma_{1, 1} \geq \Sigma_{2, 2} \ldots \geq \Sigma_{k', k'} > 0 \text{ and } \Sigma_{k' + 1, k' + 1} = \ldots = \Sigma_{m, m} = 0$ for some $k'$). A null-space signature of the manifold is defined as $T(X) = U_{k' +1:m} U_{k' +1:m}^{\top}$ where $U_{k' +1:m} \in \R^{m \times (m-k')}$ using $(k' + 1)^{th}$ to $m^{th}$ eigenvectors (eigenvectors corresponding to zero eigenvalues).
\end{definition}

The above definition also generalizes \cref{def:signature-linear-manifold} to incorporate the algebraic manifold by considering the moment of the feature mapping and \cref{def:signature-polynomial-manifold} reduces to \cref{def:signature-linear-manifold} when we consider the feature mapping $\phi(x)$ to be the identity function. 

We next show that, under suitable technical but generic non-degeneracy conditions, the signature of an algebraic manifold can uniquely identify it. The signature can also provide a way to verify membership in the manifold by computing a suitable inner product. 

\begin{proposition}\label{thm:nullspacepolyeqns}[Null space signature for a well-behaved algebraic manifold]
Let $X$ be a random variable corresponding to distribution $\mathcal D$ that lies in a $k$-dimensional well-behaved algebraic manifold of degree $\ell$. Under suitable non-degeneracy conditions\footnote{We will say that a distribution $\mathcal D$ on the manifold $\mathcal M$ is non-degenerate if for any subset of features in  $\phi(x)$ that are linearly independent over the entire manifold are also linearly independent over the support of the distribution. For example, suppose $\phi$ is the polynomial features mapping,  and the manifold is analytic. In that case, a distribution supported on a ball on the manifold is non-degenerate (this is because if a polynomial is identically zero within a ball then it is identically zero over the entire analytic manifold).} on $\mathcal D$, if we compute the signature $T(X)$ with $\phi$ being the degree $\ell$-polynomial feature mapping (i.e., $\phi(x)$ contains all monomials of $x$ of up to degree $\ell$), then the signature $T(X)$ (\cref{def:signature-polynomial-manifold}) uniquely identifies the manifold $\mathcal{M}$. That is, a point $x \in \mathcal{M}$ if and only if $\langle \phi(x) \phi(x)^\top, T(X) \rangle = 0$, and when $x \notin \mathcal M$, we have $\langle \phi(x) \phi(x)^\top, T(X) \rangle > \| P(x) \|^2$. Under suitable assumption \footnote{The assumption is that $\| \nabla P(x) \| > c$ for some constant $c$ (in a certain vicinity of interest around  the manifold.)}, this is $\Omega(d(x, \mathcal M)^2)$ where $d(x, \mathcal M) = \min_{y \in \mathcal M} \| x - y \|_2$.

\end{proposition}

The proof of \cref{thm:nullspacepolyeqns} is given in \cref{appendix-new-subsec:algebraic-manifold-proofs}.

\subsection{Extending to a generative representation of manifolds}

In \cref{subsec:well-behaved-manifold}, we considered the manifold $\mathcal M$ given by a set of polynomial equations (see \cref{def:algebraic-manifold}) and considered any distribution on such a manifold $\mathcal M$. Another way to represent a high-dimensional distribution that lies in a low-dimensional manifold $\mathcal M$ is by a push-forward function of a low-dimensional distribution. We call such a representation of the distribution as a generative representation and it is formally defined as follows:

\begin{definition}[Generative representation of $k$-dimensional manifold]
\label{def:generative-manifold}
Given a $k$-dimensional manifold in $\R^d$, we call $G:\R^k \rightarrow \R^d$ a generative representation (if it exists) of $\mathcal{M}$ if the following holds: for all $x \in \R^d$, $x \in \mathcal{M}$ if and only if there exists $z \in \R^k$ such that $x = G(z)$. Additionally, the distribution of $Z \in \R^k$ defines the distribution $\mathcal D$ of $X$ on the manifold $\mathcal M$.

\end{definition} 

Now, we show that when the function $G$ is a polynomial of degree $r$ then we can convert the generative representation of the manifold (\cref{def:generative-manifold}) to the well-behaved algebraic manifold (\cref{def:algebraic-manifold}) given by the set of polynomial equations.
This can be viewed as a form of Effective Nullstellensatz~\cite{effectivenullstellen}.
\begin{theorem}[Polynomial generative representation to algebraic manifold]
\label{thm:implicit-polynomial-manifold-degree}
    Let $X \in \mathbb{R}^{d}$ be a random variable on a $k-$dimensional manifold $\mathcal M$ that has a degree $r$ polynomial generative representation ($X = G(Z)$ where $Z \in \mathbb{R}^k$ with $G$ being a degree $r$ polynomial). Then, the manifold can be written as zero sets of $d-k$ polynomials and the degree of each polynomial is at most $(cr)^k$ for some constant $c$. That is, there exists $w_1, w_2, \ldots, w_{d-k}$ vectors such that for every $x \in \mathcal M$, $\langle w_i, \phi(x) \rangle = 0$ for all $i \in [d-k]$ for polynomial feature mapping $\phi$ of degree $(cr)^k$ for some constant $c$.
\end{theorem}

The proof makes use of a simple form of elimination~\cite{cox1997ideals} by raising the generator equation $X = G(z)$ to a sufficiently high tensor power $u$ to obtain $X^{\otimes u} = G(z)^{\otimes u}$. Note that then there will be more different monomials in $X^{\otimes u}$ than $G(z)^{\otimes u}$ because $X$ has more variables than $z$. Therefore, one can eliminate the implicit variables $z$ to obtain a polynomial representation corresponding to the manifold. See \cref{appendix-new-subsec:generative-manifold-proofs} for the complete proof.   

Next, we show that even if $G$ is not a polynomial but a general analytic function, the manifold can be approximated as $d-k$ polynomials of high enough degree. 

\begin{theorem}
\label{thm:analytic-implicit-degree}
    Let $X \in [-1, 1]^{d}$ be a random variable on a $k-$dimensional manifold $\mathcal M$ and has an analytic generative representation ($X = G(Z)$ where $G$ is analytic function and $Z \in [-1, 1]^k$) such that $\| \nabla^{(m)} G(z) \|_2 \leq 1$ for any $m$ and for any $z \in [-1 , 1]^k$. Then, the manifold can be approximately written as zero sets of $d-k$ polynomials. That is, there exists $w_1, w_2, \ldots, w_{d-k}$ vectors such that for every $x \in \mathcal M$, $| \langle w_i, \phi(x) \rangle | \leq \varepsilon$ for all $i \in [d-k]$ for polynomial feature mapping $\phi$ of degree at least $c^k \log (1 / \varepsilon) $ for some sufficiently large constant $c$. 
\end{theorem}

The proof of \cref{thm:analytic-implicit-degree} combines the result from \cref{thm:implicit-polynomial-manifold-degree} to the (Taylor) polynomial approximation of the function $G$. See \cref{appendix-new-subsec:generative-manifold-proofs} for more details.

\paragraph{Reducing the signature size using random projections.} Observe that if $\phi$ is a polynomial feature mapping of degree $\ell$, then size of $\phi$ is $O(d^\ell)$. Therefore, the null space signature $T(X)$ resulted by using the \cref{def:signature-polynomial-manifold} will at least of size $O(d^\ell)$. However, in the next result, we show that random projection can reduce the signature size to be independent of the dimension when the manifold is given by the generative representation. 

\begin{theorem}\label{thm:jl-manifold-lemma}[$k$-dim manifold generated by degree $r$ polynomials]
A $k$-dim manifold with generative representation $G$ that are polynomials of degree $r$, can be represented by a signature $T$ of size $r^{\tilde{O}(k^2)}$ (independent of ambient dimension $d$). Such a signature can be computed by projecting the points to a lower dimensional space and computing the signature as before. This signature can be used to test the membership of a point $x$ with high probability by checking the value of $\langle \phi(x) \phi(x)^T, T \rangle$: $\langle \phi(x) \phi(x)^T, T \rangle=0$ if $x$ is on the manifold. {\bf Far points: } for any point (not on the manifold), $\langle \phi(z) \phi(z)^T, T \rangle \ge \| P(x) \|^2 $ same as \cref{thm:nullspacepolyeqns}.

\end{theorem}

The proof of the above theorem can be found in \cref{appendix-new-subsec:generative-manifold-proofs}. 

We only prove the result for the generative representation of the manifold because of some technical conditions however, we believe that it might be true for well-behaved manifolds as well but we leave it as a future direction.

\section{Learning architecture augmented with a table of concept signatures}
\label{sec:concept-signature}

In \cref{sec:concent-signature-intro}, we showed that the null space signature uniquely identifies the manifold $\mathcal M$. One can consider that a manifold corresponds to a concept (for example, the manifold of a circle, square, etc.).  In this section, we propose an algorithm and architecture that computes the null space signature of the concepts in a hierarchical manner. 

Our proposed architecture is composed of mainly two modules - a concept discovery module and a concept storage module. The concept discovery module mainly computes the null space signature on the current input whereas the concept storage module stores a dictionary of non-linear transformations of \textit{latest} $m$ of signatures $T_{\ell, 1}, T_{\ell, 2} \ldots, $ at each layer $\ell$ based on recent past inputs (\cref{figure2}). This concept storage module can be thought of as a queue data structure of the last $m$ signatures. The attention at layer $\ell$ in our concept discovery module attends to the stored concept at layer $\ell$ and combines top $K$ similar concepts stored at layer $\ell$ to obtain the concept at layer $\ell+1$. Here, the value of $K$ is a hyperparameter of the algorithm. Using this idea, we propose the following learning method (with notation: $T_{\ell, i}$ being null space signature, $S_{\ell, i}$ being non-linear transformation of $T_{\ell, i}$ except at the first layer).

\begin{enumerate}
\item Given an input vector $x_t$ at time $t$. The concept storage module at the first layer keeps track of the feature mappings of the latest $m$ input vectors (i.e. $S_{1, i} = \phi_1( x_{t + 1 - i} ) \phi_1( x_{t + 1 - i} )^\top$, $i=1,..,m$).
\item The first step in computing the signature for the layer $\ell + 1$ given signatures at layer $\ell$ is to compute the attention of the latest $S_{\ell, 1}$ to the stored signatures $\{ S_{\ell, i} \}_{i=1}^m$ using cosine-based similarity (defined below) and obtain the $S_{\ell, i_1},  S_{\ell, i_2}, \cdots, S_{\ell, i_K}$ with the highest $K$ attention score.
\item In the second step of computing the signature for the layer $\ell + 1$, we compute the null space of $\frac{1}{K}\sum_{j=1}^K S_{\ell, i_j}$ (average of the signatures with the highest $K$ attention). We call the null space $T_{\ell+1, 1}$.
\item In the third step, we apply the feature mapping $\phi_{\ell + 1}$ on the null space $T_{\ell+1, 1}$ computed in the previous step, set it as $S_{\ell+1, 1}$ (i.e., $S_{\ell+1, 1} = \phi_{\ell + 1}(T_{\ell+1}$)) and store it in the concept storage module. Recall that the concept storage module behaves like a queue data structure. Hence, it will store $S_{\ell+1, 1}$ and remove the oldest signature at layer $\ell + 1$ (signature corresponding to smallest $t$).

\label{item:architecture}
\end{enumerate}

\begin{figure*}[ht]

\centering
\includegraphics[scale=0.31]{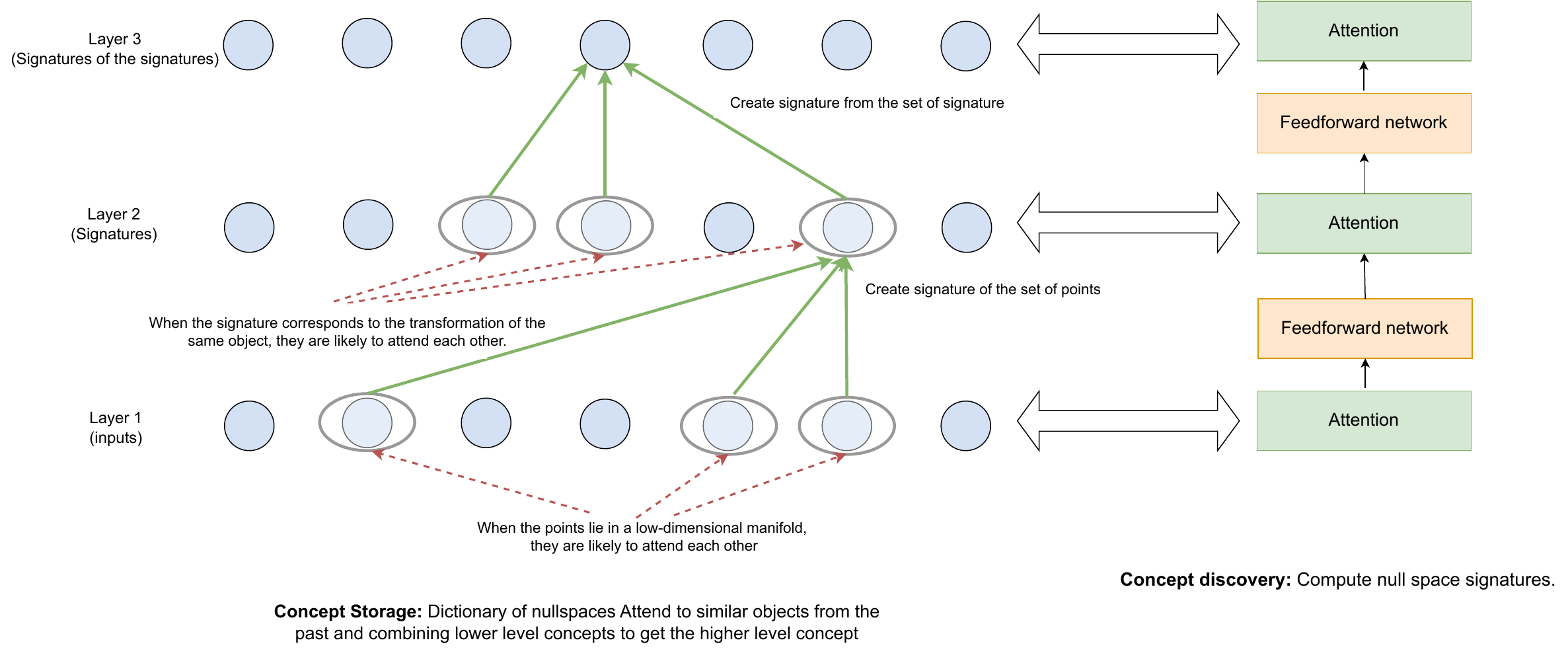}
\caption{Overview of our learning architecture.}
\label{figure2}
\end{figure*}

The above process keeps repeating $L$ times. Now, we define the cosine-based similarity and show that if the $k$-dimensional manifold has a small distortion then points on the manifold are likely to have higher cosine similarity.

\paragraph{Cosine similarity based attention.} In the above algorithm, we use attention based on the cosine similarity ($\text{attention between $x$ and $y$} = \frac{\langle x, y \rangle}{\| x\|_2 \|y\|_2}$). In the following result \cref{thm:attentiondistance}, we prove that the cosine similarity between points from a low-dimensional manifold is likely to be high.  

\begin{proposition}\label{thm:attentiondistance}[Attention based on cosine similarity]
Given a $k$ dimensional manifold $X$ with constant distortion, for every $\varepsilon > 0$, $N \in \mathbb{Z}^+$, for every set of $\Omega( N(\log(k)/\varepsilon)^k)$ points in set $\{x \mid \|x\|_2 \leq 1 \}$, there must be a set $S$ of $N$ points such that for every $x, x' \in S$, $ \frac{\langle x, x' \rangle}{\| x \|_2 \|x' \|_2} \geq 1- \varepsilon$.     
\end{proposition}

The above result (~\cref{thm:attentiondistance}) shows that points lying in a low dimensional manifold are likely to have a very high cosine similarity. Therefore, if there is a set of $K-1$ points $x_{i_1}, x_{i_2}, \cdots, x_{i_{K-1}}$ and $x_t$ that lie in the same manifold, then the null space of $\frac{1}{K}\sum_{j=1}^K S_{1, i_j}$ is more likely to be the signature of that manifold. Similarly, at layer $\ell$, $T_{\ell +1, 1}$ will more likely be a signature for the low dimension manifold that is spanned by a subset of $T_{\ell, i_1},  T_{\ell, i_2}, \cdots, T_{\ell, i_K}$. Thus, if there is a subset of signatures in $\{T_{\ell, 2}, T_{\ell, 3}, ..., T_{\ell, m} \}$ that lies on a low dimension manifold that contains $T_{\ell, 1}$, then  $T_{\ell +1, 1}$ is more likely to be the signature of that manifold.

\section{Extracting structure across the concepts}
In \cref{sec:concept-signature}, we proposed an algorithm using the null space signature and showed how cosine similarity-based attention can be used to group points together when the points lie in a low-dimensional manifold. In this section, we show how the signature of similar concepts can be combined to obtain the structure across the concepts. 

We first show that the cosine similarity of the signatures of similar concepts is high therefore they can be grouped together. Formally, we show the following result.

\begin{lemma}[Similar Manifolds] \label{lemma:similar-manifolds}
Let $\mathcal D_1$ and $\mathcal D_2$ with random variables $X_1$ and $X_2$ be two distributions on $(d-1)$-dimensional well-behaved algebraic manifold $\mathcal M_1$ and $\mathcal M_2$ of degree $\ell$ such that for any point $x$ in the support of $\mathcal D_1$, $\langle w_1, \phi(x) \rangle = 0$ (similarly for any $x$ in support of $\mathcal D_2$, $\langle w_2, \phi(x) \rangle = 0$). Then, the inner product between their null space signature $\langle T(X_1), T(X_2) \rangle = \langle w_1, w_2 \rangle^2$. 

\end{lemma}

To provide an intuitive understanding, consider the case of the linear manifolds. Recall that the null space signature $T$ of a linear manifold captures the subspace of all orthogonal directions to the manifold. Therefore, the inner product between null space signatures of two concepts would be high if the subspace of orthogonal directions has a high inner product and hence, cosine similarity-based attention can group them together. 

The proof of the above result is provided in \cref{appendix-new-sec:proofs-extracting-structure-concepts}. Additionally in the appendix, we also prove that the dot product $\langle T(X_1), T(X_2) \rangle$ increases as the dimension of $U_1 \cap U_2$ increases for linear manifolds $U_1$ and $U_2$.

We next show that some simple operations on the signature of the concepts can extract useful structures across concepts. 

\subsection{Intersections of concepts.} 

We start by showing how the signature of the intersection of two concepts can be learned just from the signatures of the underlying concepts. A point is in the intersection of two concepts/manifolds if and only if that point is in both of the manifolds. Our result in this case is as follows.  

\begin{proposition}\label{thm:structuresigs}

Given two concepts with manifolds $U_1$ and $U_2$, the intersection of the two concepts is given by the intersection of the manifolds $U_1 \cap U_2$. The signature of this intersection can be computed as follows. Then, 
$$T(U_1 \cap U_2) = \text{Nullspace}( T(U_1) + T(U_2) )$$

\end{proposition}

\subsection{Concepts that are union of a few simpler concepts}

In this section, we consider the setting where there are $m$ concepts $U_1, U_2, \ldots, U_m$ but each concept $U_i$ is constructed by taking the union of a constant number of concepts from the set of $n$ atomic concepts $A_1, A_2, \ldots, A_n$. To be precise, the null space of concept $U_i$ is obtained by taking the intersection of null space of some of the atomic concepts; note that for linear manifolds this is equivalent to taking the manifold spanning the two union of the two manifolds. In this case, we show that given only access to the null space signature of $U_1, U_2, \ldots, U_m$, we can recover the null space signature of $A_1, A_2, \ldots, A_n$. Our result is as follows.

\begin{lemma}\label{thm:dictionarysigs}[Informal]
Let $m$ be sufficiently larger than $n$ and let $U_1, U_2, \ldots, U_m$ satisfy some non-degeneracy condition, then given access to only $T(U_1), T(U_2), \ldots, T(U_m)$, there exist an operation on $T(U_1), T(U_2), \ldots, T(U_m)$ that recovers the null space signatures of the atomic concepts (i.e., it recovers $T(A_1), T(A_2), \ldots, T(A_n))$.
\end{lemma}

The proof of the above result is provided in \cref{appendix-new-sec:proofs-atomic-concepts} and it uses \cref{thm:structuresigs} to separate the atomic concepts $A_1, A_2, \ldots, A_n$ by taking intersections of $U_1, U_2, \ldots, U_m$.

\subsection{Example of signature of a circle and general concept of circles}

We provide several examples in \cref{appendix-new-sec:examples} to illustrate the potential of using the null space signature to obtain desired properties such as higher-level concepts from lower-level concepts and obtaining rotation or translational-invariant signatures, etc. To be more precise, we provide an example in \ref{appendix-new-sec:examples} of obtaining the signature of a circle and extend it to show that one can get the signature of the general concept of concentric circles by combining several signatures of circles of different radii. We also show how we can obtain rotational and translation invariant signatures of a manifold in the Appendix. Assuming two manifolds go through some motions then we show how we can obtain the common motion pattern in the Appendix.

\section{Connection to Transformer architecture}
\label{sec:transformer}

Our learning algorithm in \cref{sec:concept-signature} uses the attention mechanism that groups points from the same concept together, MLP layer that computes the signature of the concept and a memory table that keeps track of moment signature seen so far. Even though we presented the signatures as matrices, we will argue how this signatures will arise from flattened vectors from random MLP layers with small amount of trained projection. 

On a high level, we envision that a feedforward network can contain the signature $S(x) = \phi(x) \phi(x)^\top$ for each token and attention can be used to combine tokens from the same concept together to obtain $M(X) = \frac{1}{|X|} \sum_{x \in X} S(x)$ and the subsequent feedforward network can be used to obtain $T(X)$ using $M(X)$ because $T(X)$ is the null space of $M(X)$ and it can be obtained by taking by repeated multiplication of $(I-M(X))$. See \cref{fig:transformer-connection} for the high-level overview of the connection and see \cref{appendix-new-sec:proofs-connections-transformers} for the detailed discussion. 

Precisely, 
\begin{itemize}
    \item Given an input $x$, when passed through a randomly initialzed feedforward layer represents (by applying an appropriate projection) the flattened view of the signature $S(x) = \phi(x) \phi(x)^\top$ (Note that neural networks are known to perform kernel transformation \cite{allen2018learning,allen2018convergence,jacot2018neural})
    \item Next, an attention layer with identity key-query-value matrix (under appropriate scaling) performs the aggregation $$M(X) = \frac{1}{|X|} \sum_{x \in X} S(x).$$  
    \item Finally, additional feedforward layers would easily perform the power iteration on $I-M(X)$ to give (the flattened view of) the null space signature $T(X)$.
\end{itemize}

\section{Experiments}
\label{appendix:experiments}

In this section, we show the effectiveness of keeping the concept storage unit to represent and manipulate the concepts at different layers using synthetic data. 


\begin{table*}[ht]
    \centering
    \begin{tabular}{c c c c} 
     \toprule
     \textbf{Dataset} \;\;\; & \textbf{Layer} & \textbf{Percentage} (\%)  &  \textbf{Percentage} (\%) \\ [0.5ex] 
     & & (Signature storage=4000) & (signature storage=5000)  \\
     \midrule
     4-circles & \multirow{3}{*}{Layer 2} & 98.8 & 99.0  \\ [0.5ex]
     6-circles &  & 96.2 & 98.0 \\[0.5ex]
     8-circles &  & 96.0 & 97.6  \\ [0.5ex]
     \midrule
     4-circles+4-parabolas & \multirow{3}{*}{Layer 2} & 97.9 & 98.3 \\ [0.5ex]
     6-circles+6-parabolas & & 97.2 & 98.9 \\ [0.5ex]
     8-circles+8-parabolas & & 96.9 & 97.2 \\
     \midrule
     4-circles+4-parabolas & \multirow{3}{*}{Layer 3}& 99.5 & 99.7 \\ [0.5ex]
     6-circles+6-parabolas & & 99.4 & 99.3 \\ [0.5ex]
     8-circles+8-parabolas & & 98.9 & 98.7 \\
     \bottomrule
    \end{tabular}
    \caption{Percentage of signatures from the same concept in the top $K$ signatures with high attention.}
    \label{tab:exp}
\end{table*}

\paragraph{Experimental setup.} We experiment using two synthetic datasets -- 1) $n$-circles and 2) $n$-circles+$n$-parabolas. The $n$-circles dataset contains $n$ concentric circles of radius $\{1, 2, \ldots, n\}$. The $n$-parabolas dataset contains $n$ parabolas with the same axis of symmetry but the shift along the axis is $\{0.5, 1.5, \ldots, n - 0.5 \}$. The $n$-circles+$n$-parabolas contain $2n$ shapes ($n$ from each of the circle and parabola) and the number of points from each shape is equally divided. In our synthetic dataset, a circle with a given radius (similarly parabola with a given shift) denotes a lower-level concept and all circles with different radius (or parabolas with different shifts) are part of a higher-level concept of the circle shape (or the parabola shape). Note that we chose the shifts of the parabola such that there is an intersection/overlap between the circles and the parabola. See \cref{fig:toy-dataset} for the example of the 4-circles + 4-parabolas dataset.

As the novelty of our work from the architecture perspective mainly lies in manipulating concepts from the concept discovery module, we focus on experimenting with the concept discovery module with the predefined features $\phi( \cdot )$. For all our experiments, $\phi(x)$ contains all monomials up to degree 2. Then, we pass $\phi(x)$ as an input to our learning architecture mentioned in \cref{sec:concept-signature}. We use cosine similarity-based attention on the signatures to combine signatures from the same concept and obtain a higher-level signature. 

In experiments of $n$-circles dataset, we use concept storage unit with 2 layers and in experiments of $n$-circles+$n$-parabolas dataset, we use the concept storage unit with 3 layers. In our experiments, we use $K=5$. We vary the signature storage size in $\{4 \times 10^3 , 4.5 \times 10^3, 5 \times 10^3, 5.5 \times 10^3 \}$ and the number of concept $n$ in $\{4,6,8\}$. We used 100 test samples to calculate the metrics. The experimental results are averaged over 5 independent iterations.

\begin{figure}
    \begin{subfigure}[t]{0.5\textwidth}
        \centering
        \includegraphics[width=0.56\linewidth]{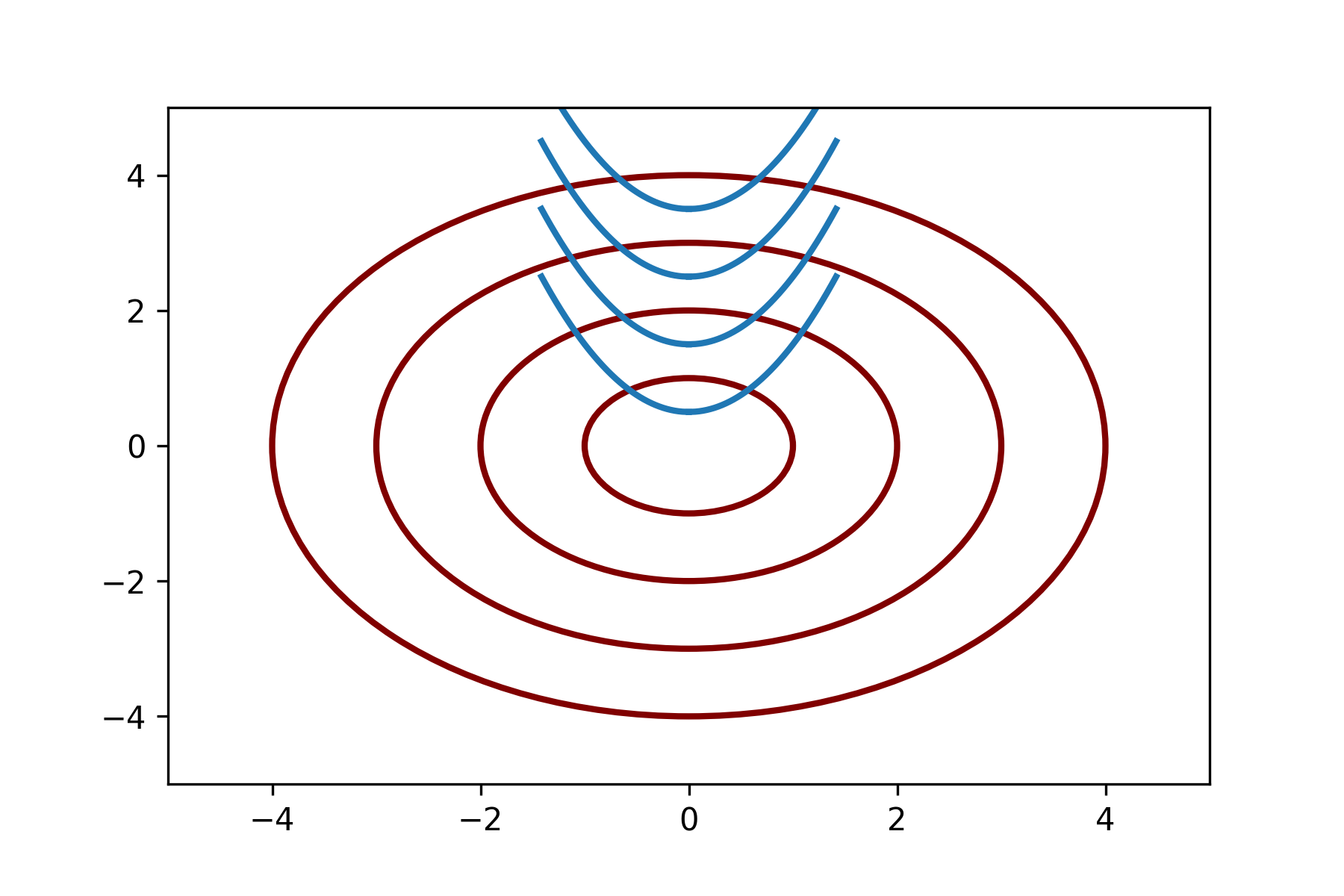}
        \caption{4-circles + 4-parabolas dataset. Each circle corresponds to a lower-level concept. All four circles are part of a general higher-level concept of circles (similarly for parabola).}
        \label{fig:toy-dataset}
    \end{subfigure} \hspace{5mm}
    \begin{subfigure}[t]{0.5\textwidth}
        \centering
        \includegraphics[width=0.56\linewidth]{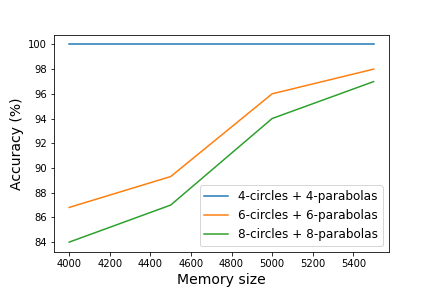}
        
        \caption{We plot the percentage of top-$K$ signatures with the highest attention which corresponds to the signature of a general circle for a new unseen circle signature.}
        \label{fig:acc_general_signature}
    \end{subfigure}
    \caption{ 
    Synthetic dataset of type $n$-circles+$n$-parabola and experimental findings on it.}
    \label{figure-experiments}
\end{figure}

\paragraph{Questions and evaluation.} We perform experiments to answer the following three questions: 
\begin{enumerate}
    \item To obtain a higher-level concept signature, do signatures only attend to the signatures from the same concept? To answer this question, we measure the percentage of signatures from the same concept in the top $K$ signatures for each data point at all layers in the concept discovery unit where $K$ is the number of signatures to attend to and we present experimental results in \cref{tab:exp}.  
    \item How does the performance change as we increase the number of concepts? To answer this question, we measure the percentage of the same concept signatures during the attention by increasing $n$ in both synthetic data. We present the findings of our experiments in \cref{tab:exp}. \label{item:exp2}
    \item Can a higher-level signature of a concept generalize to unseen concepts from the same lower-level concepts? To be specific, in the $n$-circles + $n$-parabolas dataset, the general signature of a circle is created using signatures of circles with radius $\{1,2,\ldots,n\}$. Now, we ask the following question: does the signature of a circle with a radius outside of the training radius set have high attention to the general signature of a circle than the general signature of the parabola? To answer this question, we randomly sample a radius from $[1, 10]$ and create a signature of the circle with this radius. Then, we find the top-$K$ signatures with the highest attention in layer 3 and calculate the percentage of them that correspond to the signature of a general circle and report it in \cref{fig:acc_general_signature}.
    \label{item:exp3}
\end{enumerate}

\paragraph{Results.} In \cref{tab:exp}, we see that our proposed architecture only attends to the signatures from the same concept class not only in the second layer but also in the third layer. This shows that signatures of circles with different radius only attends to each other and creates a general signature of the circles with different radius (See \cref{figure2} for details about the architecture). We also see that increasing the number of concepts by increasing $n$ only mildly degrades the performance. \cref{fig:acc_general_signature} shows that having a modular architecture for each concept helps in associating a new unseen lower-level concept to the correct higher-level concept and quickly learning about the new concept.

\newpage

\bibliographystyle{alpha}
\bibliography{main}

\newpage

\appendix

\section{Subspace signature of concepts}
\label{appendix-new-sec:subspace}

In this section, we provide the proofs for \cref{sec:concent-signature-intro}. We start with the proof of signature for the linear manifold case. 

\subsection{Linear manifold}
\label{appendix-new-subsec:linear-manifold-proofs}

We restate the proposition of the linear manifold case as below. 

\begin{proposition}
\label{prop:linear-manifold-appendix}
    The null space signature $T(X)$ of a $k$-dimensional linear manifold $\mathcal M$ (defined in \cref{def:linear-manifold}) uniquely identifies the manifold of the distribution $\mathcal D$. That is, for any point $x$ on the manifold $\mathcal M$, $\langle T(X), x x^\top \rangle = 0$ and for any point $x$ not on the manifold $\mathcal M$, $\langle T(X), x x^\top \rangle > 0$ and equals the distance of $x$ from the manifold.
\end{proposition}

\begin{proof}
    As explained in the proof sketch, $T(X)$ is a null space of the moment matrix and therefore only contains orthogonal directions of the manifold. This gives us that for all $k' > k$, $U_{k'}^\top x = 0$. Hence, if $x \in \mathcal M$, $\langle T(x), x x^\top \rangle = \langle U_{k+1:d} U_{k+1:d}^\top, x x^\top \rangle = 0$. If $x \notin \mathcal M$, then it means that $x$ has some component in one of the orthogonal directions of the manifold. This implies that there exists $k' > k$ such that $U_{k'}^\top x > 0$ which implies that $\langle T(x), x x^\top \rangle = \langle U_{k+1:d} U_{k+1:d}^\top, \; x x^\top \rangle > 0$.
\end{proof}

\subsection{Well-behaved algebraic manifold}
\label{appendix-new-subsec:algebraic-manifold-proofs}

Now, we will provide the proof for the signature of the well-behaved algebraic manifold.

\begin{proposition}
    Let $X$ be a random variable corresponding to distribution $\mathcal D$ that lies in a $k$-dimensional well-behaved algebraic manifold of degree $\ell$. Under suitable non-degeneracy conditions\footnote{We will say that a distribution $\mathcal D$ on the manifold $\mathcal M$ is non-degenerate if for any subset of features in  $\phi(x)$ that are linearly independent over the entire manifold are also linearly independent over the support of the distribution. For example, suppose $\phi$ is the polynomial features mapping,  and the manifold is analytic. In that case, a distribution supported on a ball on the manifold is non-degenerate (this is because if a polynomial is identically zero within a ball then it is identically zero over the entire analytic manifold).} on $\mathcal D$, if we compute the signature $T(X)$ with $\phi$ being the degree $\ell$-polynomial feature mapping (i.e., $\phi(x)$ contains all monomials of $x$ of up to degree $\ell$), then the signature $T(X)$ (\cref{def:signature-polynomial-manifold}) uniquely identifies the manifold $\mathcal{M}$. That is, a point $x \in \mathcal{M}$ if and only if $\langle \phi(x) \phi(x)^\top, T(X) \rangle = 0$, otherwise $\langle \phi(x) \phi(x)^\top, T(X) \rangle > \| P(x) \|^2$. Under assumption that $\| \nabla  P(x)  \| > c$ for some constant $c$, 
     $\| P(x) \|^2 \geq \Omega(d(x, \mathcal M)^2)$ . 
\end{proposition}

\begin{proof}
Let $p:\R^d \rightarrow \R$ be a polynomial that vanishes on the manifold $\mathcal{M}$. Then, for any point $x$, $p(x) = \langle w, \phi(x) \rangle$ (here we view $w$ as the coefficient vector of the polynomial $p(x)$). Thus, if $p(x) = 0$ for all $x \in \mathcal{M}$, then $w$ is in the null-space of $M(X)$ as $w^T \phi(x_i) \phi(x_i)^T w = 0$ for all $x_i$.

Additionally, for some $x$, if there exists $w$ such that $\langle w, \phi(x) \rangle > 0$. That implies that all the polynomial factors of the polynomial $\langle w, \phi(x) \rangle$ are non-zero. However, if this happens then it means that one of the polynomials (of $\langle w_i, \phi(x) \rangle$) that define the manifold is non-zero. Therefore, it means that $x$ does not lie on the manifold $\mathcal M$. Further, note that the coefficient vectors $w_1, \ldots, w_{d-k}$ form an orthonormal basis of a subspace that lies within the null space. This implies that $\langle \phi(x) \phi(x)^\top , T(X) \rangle \geq \langle  \phi(x) \phi(x)^\top , WW^\top \rangle = \sum_{i=1}^{d-k} P_i(x)^2 = \| P(x) \|^2$ where $i$th column of $W$ is $w_i$.  Further, if $\| \nabla P(x) \|$ is lower bounded, then look at the path of length say $l$ of steepest gradient flow from $x$ to the manifold; $|P(x)| \ge cl$. But $l$ can only exceed  the shortest path from $x$ to the manifold. Thus $\| P(x) \| \geq c d(x, \mathcal M)$.

To generalize the result for the null space signature when the moment $M(X)$ is computed, first observe that as long as the distribution is bounded or sub-exponential, then computing $M(X)$ with $\text{poly}(d/\epsilon)$ samples will result in $\epsilon$ close approximate to the moment matrix $M(X)$ with high probability. Therefore, one can generalize the definition of the null space $T(X)$ to contain eigenvectors whose eigenvalue is smaller than $\epsilon$ instead of eigenvectors with zero eigenvalues. In this case, taking $\text{poly}(d/\epsilon)$ samples and assuming $\mathcal{M}$ is non-degenerate, one can get the null space signature $T(X)$ from finite samples with the same guarantees. 

\end{proof}

\subsection{Generative representations}
\label{appendix-new-subsec:generative-manifold-proofs}

In this section, we present our results when the manifold is given by the generative representation \cref{def:generative-manifold}. We first prove that when the generative representation is a polynomial and the manifold is a $k$-dimensional manifold in $k+1$ dimensional space, then we can convert the generative representation of the manifold to the well-behaved algebraic manifold representation given in \cref{def:algebraic-manifold}. 

\subsubsection{Polynomial generative representation of the manifold}
\label{appendix-subsec:explicit-poly-repr-implicit-poly}

We will first prove the case when $X$ is a $k-$dimensional manifold in $(k+1)-$dimensional space.

\begin{lemma}
\label{lemma:implicit-poly-explicit-degree} 
    Let $X \in \mathbb{R}^{k+1}$ be a random variable in a $k-$dimensional manifold $\mathcal M$ that has a degree $r$ polynomial generative presentation ($X = G(Z)$ where $Z \in \mathbb{R}^k$ with $G$ being a degree $r$ polynomial). Then, the manifold can be written as zero sets of a polynomial $H(X)$ and the degree of the polynomial is at most $r^k$. That is, there exists a $w_1$ vector such that for every $x \in \mathcal M$, $\langle w_1, \phi(x) \rangle = 0$ for polynomial feature mapping $\phi$ of degree $(cr)^k$ for some constant $c$.
\end{lemma}

\begin{proof}[Proof of \cref{lemma:implicit-poly-explicit-degree}]
    Let $u$ be the upper bound on the degree of $H(X)$. We show that there exists a polynomial constructed from the entries of $X^{\otimes u}$ such that $H(X) = 0$ for points on the manifold.

    We have $X^{\otimes u} = G(z)^{\otimes u}$ by taking $u^{\text{th}}$ tensor power of $X = G(z)$ equality. Denote the number of different monomial entries in $X^{\otimes u}$ with $X \in \mathbb{R}^{k+1}$ by $m(u, k+1)$. Observe that $X^{\otimes u} = G(z)^{\otimes u}$ has $m(u, k+1)$ different equalities from because $X^{\otimes u}$ has $m(u, k+1)$ different monomials. The right side $G(z)^{\otimes u}$ has at most $m(u \cdot r, k)$ different monomial terms, and when $m(u, k+1) > m(u \cdot r, k)$, we can eliminate the implicit variables in the remaining $m(u, k+1) - m(u \cdot r, k)$ equations. The remaining equations are consistent because $X$ is generated using $G(z)$. 

    The degree of the polynomial $\phi(X)$ is the smallest $u$ such that $m(u, k+1) > m(u \cdot r, k)$. The number of monomials of degree $u$ in $k+1$ variables is given by $m(u, k+1) = C^{u+k}_{u}$. The degree of $\phi(X)$ is smallest $u$ such that $C^{u+k}_{k} > C^{r \cdot u + k - 1}_{k - 1}$. This inequality satisfies when $u > (cr)^{k-1}$ for some absolute constant $c$. 
    \begin{align*}
        C^{r \cdot u + k - 1}_{k - 1} &\leq \frac{ e^{k-1} (r \cdot u + k -1)^{k-1} }{ (k-1)^{k-1} } \leq \frac{ (2e)^{k-1} (r \cdot u)^{k-1} }{ (k-1)^{k-1}  } \\ 
        &\leq \Big( \frac{ 2e }{ c(k-1) } \Big)^{k-1} u^k \leq \Big( \frac{ 1 }{ k^k } \Big) (u + k)^k \leq C^{u+k}_k
    \end{align*}
    The second inequality follows from $k - 1 \leq r \cdot u$.
\end{proof}

\begin{proof}[Proof of \cref{thm:implicit-polynomial-manifold-degree}] For extending the \cref{lemma:implicit-poly-explicit-degree} to $X \in \mathbb{R}^d$ case, we start by identifying $k$ direction along a local neighbourhood of any point on the neighbourhood. Note that this can be done by looking at the tangent direction of the manifold of the points in the neighbourhood. After determining the tangent direction, we can rotate the points so that $\{ X_1, X_2, \ldots, X_k \}$ in the neighbourhood.

\noindent Assuming $\{ X_1, X_2, \ldots, X_k \}$ covers a $k-$dimensional space, we use \cref{lemma:implicit-poly-explicit-degree} for each of $\{X_{k+1}, \ldots, X_d \}$ as $k+1$-th coordinate and obtain $d-k$ polynomials of the degree $(cl)^k$ for some absolute constant $c$. Note that each of the polynomials contains a feature mapping $H(X)$ that depends on the first $k$ coordinates of $X$ and one of the last $d-k$ coordinates such that $H(X) = 0$ for points on the manifold. 
\end{proof}

Now, we will show that when the generative function $G$ of the manifold is an analytic function, then also we can approximately obtain the well-behaved algebraic manifold representation.  

\subsubsection{Analytic generative representation of the manifold}
\label{appendix-subsec:explicit-poly-repr-analytic}

We will first prove for the case when $X$ is a $k-$dimensional manifold in a $(k+1)-$dimensional space given by the analytic function.

\begin{lemma}
\label{lemma:analytic-implicit-degree-kp1}
    Suppose the data $X \in [-1, 1]^{k+1}$ lies in a $k-$dimensional manifold by equation $X = G(z)$ where $G : [-1, 1]^k \to [1, 1]^{k+1}$ is an analytic function with bounded gradient $\| \nabla^{(m)} G(x) \| \leq 1$. Then, for a sufficiently large constant $c$, there exists a polynomial in $X$, $H(X)$ of degree $c^k \log (1/\varepsilon)$ such that for points on the manifold, $| H(X) | \leq \varepsilon$ for any $\varepsilon > 0$. 
\end{lemma}

\begin{proof}[Proof of \cref{lemma:analytic-implicit-degree-kp1}]
    Since $G(z)$ is an analytic function with the bounded derivatives so we can approximate $G(z)$ using a degree $q$ Taylor approximation $G'(z)$ within the ball of a unit radius of $z$ with the approximation error $\frac{1}{q!}$. Now, instead of Taylor expanding each term in $G(z)$, we can expand each term in $G(z)^{\otimes \ell}$. Suppose for any multi-index $\alpha = (\alpha_1, \ldots, \alpha_\ell)$, the Taylor approximation of $G_{\alpha}(z) = \prod_{i=1}^\ell G_{\alpha_i} (z)$ of degree $q$ is given by $\Tilde{G}_{\alpha}(z)$ as follows:
    \begin{align*}
        \Tilde{G}_{\alpha}(z) = G_{\alpha}(\mathbf{0}) + \sum_{i=1}^{q-1} \frac{ \langle \nabla^{(i)} G_\alpha( \mathbf{0} ), z^{.i} \rangle }{i!}.
    \end{align*}
    where $\mathbf{0}$ represents the zero vector. Using the remainder theorem of the Taylor series, we have
    \begin{align*}
        | \Tilde{G}_{\alpha}(z) - G_{\alpha}(z) | \leq \frac{ \| \nabla^{(q)} G_\alpha(z') \| }{ q! } \leq \frac{\ell^q}{q!}.
    \end{align*}
    for some $z'$, where the last inequality follows from the fact that the $i^{\text{th}}$ derivative of $\Tilde{G}_{\alpha}(z)$ for any $\alpha$ can be written as a sum of $\ell^i$ terms with each term bounded by 1. Hence, we can approximate each term in $G(z)^{\otimes \ell}$ with $\Tilde{G}$ within the ball of the unit radius with the approximation error $\ell^q / q!$ where each element of $\Tilde{G}(z)^{\otimes u}$ is a degree $q-$polynomial obtained by Taylor expansion.\\[3pt] 
    Recall that the $m(\ell, k)$ is defined as the number of different monomials of degree $\ell$ in $k$ variables. The total number of variables in $\Tilde{G}(z)$ is given by $m(q, k)$ and the total number of different equations in $X^{\otimes u} = G(Z)^{\otimes u}$ is $m(u, k+1)$. The degree of polynomial $H(X)$ is the smallest $\ell$ such that $m(u, k+1) > m(q, k)$. This inequality reduces to $C^{u+k}_k > C^{q+k-1}_{k-1}$. Choosing $u = c^{k-1} q^{\frac{k-1}{k}}$ for some absolute constant $c$, then this inequality holds. 
    \begin{align*}
        C^{q+k-1}_{k-1} &\leq \Big( \frac{ e(q+k-1) }{(k-1)} \Big)^{k-1} \leq \frac{u^k}{k^k} \\ 
        &\leq \frac{ (u+k)^k}{ k^k } \leq C^{u+k}_k
    \end{align*}
    
    Next, we show that the polynomial is close to zero for the points on the manifold. The polynomial obtained by a Taylor expansion of the analytic function $G$ with the truncation at $q^{th}$ degree has approximation $\frac{u^q}{q!}$. Putting the value of $q = \frac{u^{k/(k-1)}}{c^k}$, we obtain the approximation error to be $\frac{u^q}{q!} \leq \frac{c^q u^q}{q^q} = c^q u^{-\frac{q}{k}} = u^{-\frac{ u^{(1 + 1/k) } }{2 c^k k}}$. Now, taking the degree $u \geq c^k \log \frac{1}{\varepsilon}$, we have
    \begin{align*}
        -\frac{ u^{(1 + 1/k) } }{2 c^k k} \log u \leq - \frac{ u }{2 c^k k} \leq \log \varepsilon.
    \end{align*}
    Using this inequality, we obtain the bound. 
\end{proof} 

\begin{proof}[Proof of \cref{thm:analytic-implicit-degree}]
    Similar to proof of \cref{thm:implicit-polynomial-manifold-degree}, we start by rotating the coordinates at each point such that the first $k$ coordinates $\{ X_1, X_2, \ldots, X_k \}$ in the neighbourhood covers $k$ dimensional space. Then, we use \cref{lemma:analytic-implicit-degree-kp1} for each of $\{X_{k+1}, \ldots, X_d \}$ as $k+1$-th coordinate and obtain $d-k$ polynomials that satisfies the given property. 
\end{proof}

\subsubsection{Reducing the signature size by random projections}
\label{appendix-subsec:random-projection}

We also show that we can reduce the sample complexity to $r^{\tilde{O}(k^2)}$ by using a random projection to a lower-dimensional space. The following follows easily from Theorem 1.6 in \cite{Clarkson08} (the statement there is about a $k$-dimensional manifold; we can view adding an extra point as a $(k+1)$-dimensional manifold).

\begin{lemma}[Reducing sample complexity using projection, \cite{Clarkson08}]\label{lemma:randproj}
Let $\mathcal{M} \subseteq \R^d$ be a  connected, compact, orientable, differentiable manifold M with dimension $k$, constant curvature and bounded diameter. Then, for any $\epsilon, \delta > 0$ and a point $x, \mathbb{R}^d$, if $A:\R^d \rightarrow \R^m$ is a random linear map for $m = O((k + \log(1/\delta))/\epsilon^2)$, with probability at least $1-\delta$, the following hold:
\begin{itemize}
    \item For all $y, z \in \mathcal{M}$, $\|Ay - z\|_2 = (1\pm \epsilon) \|y - z\|_2$.
    \item For all $y \in \mathcal{M}$, $\|Ax - Ay\| = (1 \pm \epsilon) \|x-y\|_2$.
\end{itemize}
\end{lemma}

We use the projection lemma to obtain the following.

\begin{proof}
We first note that if there is a generative representation of a $k$-dimensional manifold by polynomials of degree $r$, then the degree of the manifold is at most $\ell = k^r$ - see \cref{thm:implicit-polynomial-manifold-degree}. Note that any manifold with generative representation of degree $r$ continues to have a generative representation of same degree after a projection. Thus, if we use a random projection to project the space to $m = O((k + \log(1/\delta))/\epsilon^2)$, then by \cref{lemma:randproj}, all distances on the manifold and the test point $x$ are preserved with probability at least $1-\delta$. 

Let us work in the projected space of dimension $m$. Now, the number of monomials involved in the monomial map is $\binom{m+\ell}{\ell}$ which is at most $(m+\ell)^m$. Since $\ell$ is at most $r^k$, the dimension of the monomial map is $r^{k m}$. Now, note that to approximate $T(X)$, we need a spectral approximation to $M(X)$. Recall that for any distribution on vectors in dimension $D$ with bounded covariance, the number of samples needed for the empirical covariance matrix to approximate the true covariance matrix is $\tilde{O}(D)$. By applying this to $\phi(x)$ (in the projected space), the number of samples needed for $M(X)$ to approximate the true moment matrix is $\tilde{O}(r^{km})$. Thus, the size of the signature is $\tilde{O}(r^{km}) = r^{\tilde{O}(k^2)}$. 

Finally, for the last part, consider the polynomial $Q(x) = \langle \phi(x) \phi(x)^T, T \rangle$ (in the projected space). Since $T$ is a projection matrix, it is not hard to argue that for $z$ uniformly random on the unit sphere,
$$E[Q(z)^2] = \Omega_r(1).$$

The lemma's final part follows from \cref{lemma:polydistance}. 
\end{proof}

We will use the following classical result about multi-variate polynomials:
\begin{lemma}[Anti-concentration for polynomials, \cite{CarberW01}]Let $P: \R^n \rightarrow \R$ be a degree at most $d$ polynomial. Let $x$ be a random point in the unit ball.  Let $\|P\|_2 = E_x[P(x)^2]^{1/2}$. Then, for any $t, \epsilon$, 
$$Pr_{x}[ |P(x) - t| < \epsilon \|P\|_2] = O(\epsilon^{1/d}).$$
\end{lemma}

The following lemma says if we start from a polynomial with non-trivial norm, then for any ball that is centered not too far from the origin, the polynomial is unlikely to be very small. It will be important that this probability goes to zero as the radius of the ball increases. 

\begin{lemma}\label{lemma:polydistance}

Let $z$ be a uniformly random point in the unit ball of $\R^k$, $Q:\R^k \rightarrow \R$ be a degree $s$ polynomial such that $E[Q(z)^2] \geq 1$. Then, for any $y \in \R^k$, if $\|y \| \leq R$, 

$$Pr[ |Q(y + \Delta z)| > \epsilon ] < (ks) \epsilon^{1/s} \alpha^{1/s},$$
where $\alpha = \max((1+ \|y\|_\infty)/\Delta, ((1+\|y\|_\infty)/\Delta)^s)$. 
\end{lemma}

\begin{proof}[Proof Sketch for \cref{lemma:polydistance}]
    Let $P(z) = Q(y + \Delta z)$. We first argue that $E_z[P(z)^2]$ is non-trivially large. Without loss of generality, let us suppose $P$ has no constant term. 
Now, suppose that $E[P(z)]^2 = \delta$. Let $P(z) = \sum_S c_S h_S(z)$ where $h_S$ are the Legendre polynomials that form an orthonormal basis for polynomials under the uniform distribution on the sphere. Thus, $\sum_S c_S^2 = \delta^2$. Now, for $z'$ uniformly random on the unit ball, let $Q(z') = P((z'-y)/\Delta)$. Therefore, $$E[Q(z')^2] = E[ (\sum_S c_S h_S((z'-y)/\Delta))^2].$$ 
However, note that the Legendre polynomials of degree $s$ are bounded: in particular for any point $x$, and $h_S$ of degree $s$, we have $|h_S(x)| \leq s^{O(s)} (\|x\|_\infty + \|x\|_\infty^s)$. Therefore, if we let $\alpha = \max((1+ \|y\|_\infty)/\Delta, ((1+\|y\|_\infty)/\Delta)^s)$, we get 
$$E[Q(z')^2] \leq (\sum_S |c_S|) \cdot s^{O(s)} \alpha \leq k^{O(s)} s^{O(s)} \alpha \delta.$$ Therefore, we get $\delta > (ks)^{-O(s)}/\alpha$. Applying the previous lemma, we get 
$$Pr_z[ |P(z)| < \epsilon ] = O((\epsilon/\delta)^{1/s}).$$
Therefore, $Pr_z[ |Q(y+ \Delta z)| < \epsilon ] < (ks) \cdot \alpha^{1/s} \epsilon^{1/s}$. Finally, note that $\alpha = \max((1+ \|y\|_\infty)/\Delta, ((1+\|y\|_\infty)/\Delta)^s)$ is a decreasing function in $\Delta$. 
\end{proof}

\section{Cosine similarity in the architecture}
\label{appendix-new-sec:proofs-proposed-architecture}

In this section, we show that the cosine-similarity between points from a low-dimensional manifold is likely to be high. 
\subsection{Proof of \texorpdfstring{\cref{thm:attentiondistance}}{Theorem~\ref{thm:attentiondistance}}}

\begin{proposition}
    Given a $k$ dimensional manifold $X$ with constant distortion, for every $\varepsilon > 0$, $N \in \mathbb{Z}^+$, for every set of $\Omega( N(\log(k)/\varepsilon)^k)$ points in set $\{x \mid \|x\|_2 \leq 1 \}$, there must be a set $S$ of $N$ points such that for every $x, x' \in S$, $ \frac{\langle x, x' \rangle}{\| x \|_2 \|x' \|_2} \geq 1- \varepsilon$.     
\end{proposition}

\begin{proof}
First, we can consider a $k$-dimensional subspace, then the unit ball in the $k$-dimensional subspace can be covered by $O( (\log(k)/\varepsilon)^k)$ rays of angles at most $\varepsilon$. By pigeon-hole principle, we complete the proof.  For a general $k$-dimensional manifold with constant distortion, we can map it back to $k$-dimensional subspace which preserves distance up to a constant multiplicative factor.
\end{proof}

\section{Extracting structure across concepts}
\label{appendix-new-sec:proofs-extracting-structure-concepts}

We first show that the cosine similarity of the signatures of similar concepts will be high. We first prove this for $(d-1)$-dimensional well-behaved algebraic manifold.

\subsection{Proof for \texorpdfstring{\cref{lemma:similar-manifolds}}{Lemma~\ref{lemma:similar-manifolds}}}

\begin{lemma}
    Let $\mathcal D_1$ and $\mathcal D_2$ with random variables $X_1$ and $X_2$ be two distributions on $(d-1)$-dimensional well-behaved algebraic manifold $\mathcal M_1$ and $\mathcal M_2$ of degree $\ell$ such that for any point $x$ in the support of $\mathcal D_1$, $\langle w_1, \phi(x) \rangle = 0$ (similarly for any $x$ in support of $\mathcal D_2$, $\langle w_2, \phi(x) \rangle = 0$). Then, the inner product between their null space signature $\langle T(X_1), T(X_2) \rangle = \langle w_1, w_2 \rangle^2$. 
\end{lemma}

\begin{proof}
If a manifold $U$ is specified by a polynomial equation of degree $\ell$ and we use $\ell^{\text{th}}$ moments, there will be exactly one eigenvector in the null space specified by $c$. Thus $T(X_1) = w_1 w_1^\top$ and $T(X_2) = w_2 w_2^\top$. So $ \langle T(X_1), T(X_2) \rangle = \langle w_1 w_1^\top, w_2 w_2^\top \rangle = \langle w_1, w_2 \rangle^2$

\end{proof}

Now, we show that the cosine similarity between two similar low-dimensional linear manifolds is much high than two random manifolds.

\begin{lemma}
    Given two concept manifolds $\mathcal M_1$ and $\mathcal M_2$ with random variable $X_1$ and $X_2$, the intersection of the two concepts has null space given by the span of the  union of the null spaces of $X_1$ and $X_2$. Define $F(X) = I - T(X)$. Further if they are linear manifolds that are subspaces of dimension $k$ then the similarity between their signatures increases with the dimensionality $\text{dim}(X_1 \cap X_2)$ of their intersection.
\begin{enumerate}
    \item 
    $$ \langle F(X_1), F(X_2) \rangle \ge \text{dim}(X_1 \cap X_2) $$ 

    Note that  $ \langle T(X_1), T(X_2) \rangle =  \langle F(X_1), F(X_2) \rangle + d - \text{dim}(X_1) - \text{dim}(X_2).$ 

    \item In contrast two random subspaces $\mathcal M_1$ and $\mathcal M_2$ of dimension $k$ in $d$-dimensional space with random variables $X_1$ and $X_2$ have a small dot product if $k^2 << d$.
$$E[ \langle F(X_1), F(X_2) \rangle ] =  k^2/d$$

    \end{enumerate}
\end{lemma}

\begin{proof}
    let $v_1,..,v_{k_1}$ be a basis for the manifold $\mathcal M_1$ subspace and $w_1,..,w_{k_2}$ be a basis for $\mathcal M_2$ subspace. Then $F(X_1)= \sum v_i v_i^\top$ and $F(X_2)= \sum w_i w_i^\top$. Thus, $\langle F(X_1), F(X_2) \rangle =  \sum_{i,j} \langle v_i, w_j \rangle^2$. If $X_1$ and $X_2$ intersect in $\text{dim}(X_1 \cap X_2)$, then we can find basis that share $\text{dim}(X_1 \cap X_2)$ basis vectors. This gives   
$$ \langle F(X_1), F(X_2) \rangle \ge  \textrm{dim}(X_1 \cap X_2).$$

For random $v_i$ and $w_j$, $E [ \langle v_i , w_j \rangle ] = \frac{1}{d}$ which proves the result for random subspaces $X_1$ and $X_2$ (i.e., $E[ \langle F(X_1), F(X_2) \rangle ] =  k^2/d$). The result for $T(X)$ follows from the fact that $T(X) = I - F(X)$. 

\end{proof}

\begin{corollary}\label{thm:}[Similar Manifolds]

Consider random hyperplanes and spheres in $d$ dimensions where all coefficient's of hyperplanes are chosen as normal random variables and for spheres the  coordinates and radius chosen as unit normal random variables. By looking at the similarity $\langle T(U_1), T(U_2) \rangle$ we get

1. The expected similarity between two random lines is $1/d$

2. Between two parallel lines is $1 - O(1/d)$

3. Between two random spheres is $1/5$

4. Between two concentric spheres is $1$

5. Between a random line and a random sphere is $O(1/d)$
\end{corollary}

\subsection{Intersections of concepts}

Given two concepts with manifolds $U_1$ and $U_2$, the intersection of the two concepts is given by the intersection of the manifolds $U_1 \cap U_2$. Note that the null space of the $U_1 \cap U_2$ would be the span of the union of the null spaces of $U_1$ and $U_2$. Similarly one can define the union $U_1 \cup U_2$  as the concept whose null space is the intersection of the null spaces of $U_1$ and $U_2$; thus $U_1 \cup U_2$ is the concept with properties (specified in the form of polynomial equations) that is the intersection of the properties of $U_1$, $U_2$. 

\subsubsection{Proof of \texorpdfstring{\cref{thm:structuresigs}}{Theorem~\ref{thm:structuresigs}}}
 The main idea is that for two PSD matrices $A$ and $B$, \text{null-space}$(A \cap B) = \text{null-space}(A + B)$ -- this is because $x^\top (A+B) x = x^\top A x + x^\top B x$ is zero if and only if $x^\top A x = 0$ and $x^\top B x = 0$.

\subsection{Concepts that are union on a few simpler concepts}
\label{appendix-new-sec:proofs-atomic-concepts}

\subsubsection{Proof for \texorpdfstring{\cref{thm:dictionarysigs}}{Theorem~\ref{thm:dictionarysigs}}}
\label{appendix:subsec-atomic-signatures}

\begin{proof}
 This follows from the fact that the signature of the intersection of two concepts can be obtained from the signatures of the individual concepts ~\cref{thm:structuresigs}. By repeated taking intersections we get the set of atomic concepts.
\end{proof}

\subsubsection{Identifying points from the moment statistics}

The following remark demonstrates how a higher-level concept obtained by combining a few simpler concepts can be represented by the signature of the signatures of the underlying concepts; for example, a rectangle is obtained by a union of four line segments and the human stick figure would consist of six curves.

\begin{remark}
A stick figure human would consists of about 6 curves drawn on a 2d plane.  At the first level we would get signatures for each of those curves.  at the second level we would get a signature of the curve signatures.  This would become the signature of the entire stick figure. Even though the total degrees of the freedom here is large, one can easily identify the common invariant part of the human stick figure concept from just a few examples. 
\end{remark}

\begin{lemma}[Moment statistics memorize a small set of points]
For $k$ points, the signature obtained by the $O(k)$th moment statistic uniquely identifies exactly the set of $k$ points.
\end{lemma}
\begin{proof}
This follows from the fact that there is a polynomial of degree $2k$ whose zero's coincide with the set of given $k$ points. This polynomial must be in the null space of the moments statistic matrix.
\end{proof}

\subsection{Examples of the signatures}
\label{appendix-new-sec:examples}


\begin{figure*}[t]
    \begin{center}
        \begin{tabular}{cc}
            \begin{subfigure}[t]{0.495\textwidth}
                \centering
                \includegraphics[height=3cm, scale=0.18]{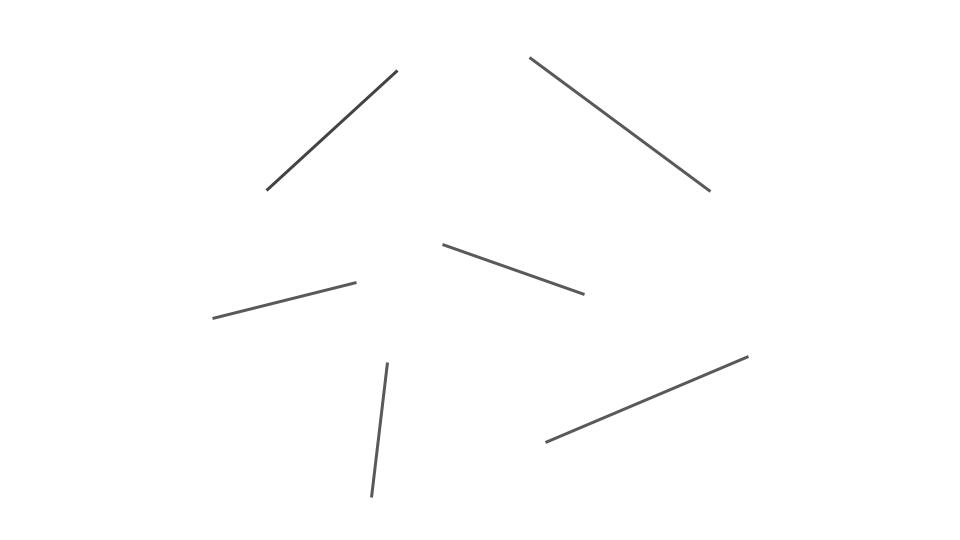}
                \caption{Several lines on a 2-d plane, each with a separate signature. \hspace{100cm} \\ \\ }
                \label{fig:line-signature}
            \end{subfigure} & 
            \begin{subfigure}[t]{0.495\textwidth}
            \centering
                \includegraphics[width=0.7\linewidth]{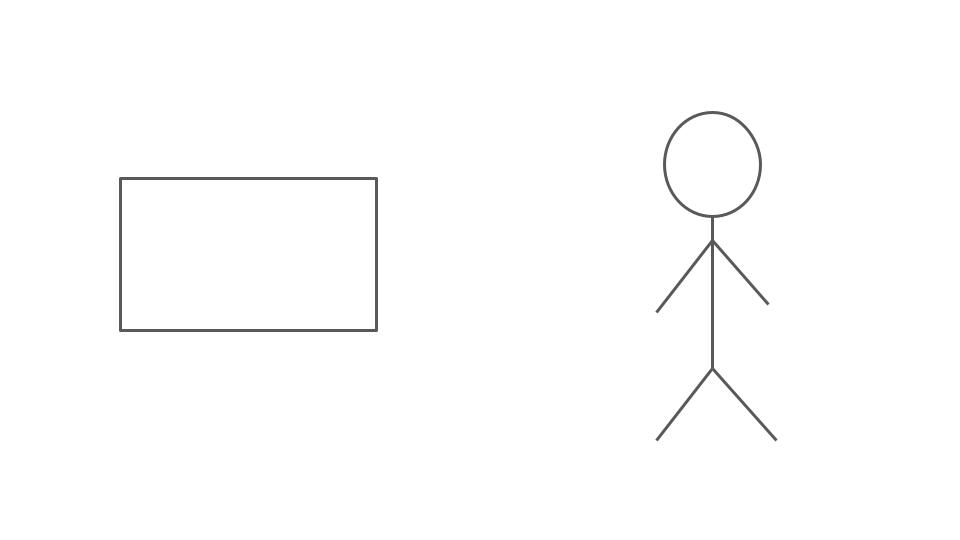}
                \caption{A rectangle with 4 segments whose signatures are related. The signatures of these four signatures form a signature of the full rectangle. Similarly for the person.}
                \label{fig:rectangle-signature}
            \end{subfigure}\\
            \begin{subfigure}[t]{0.495\textwidth}
            \centering
                \includegraphics[width=0.4\linewidth]{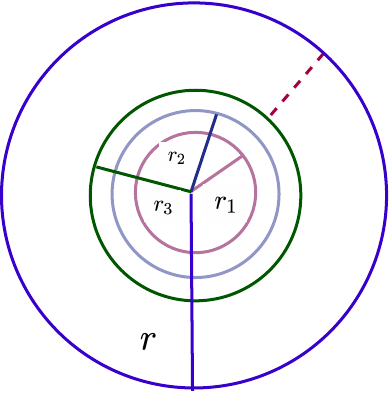}
                \caption{Several concentric circles of varying radius: the signatures of the individual circles all lie on a manifold that is the signature of the concept of a circle.}
                \label{fig:concentric-circles-signature}
            \end{subfigure} & 
            \begin{subfigure}[t]{0.495\textwidth}
            \centering
            \includegraphics[width=0.7\linewidth]{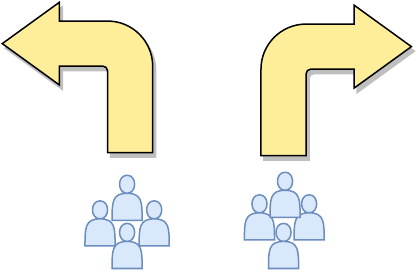}
                \caption{Two groups each traveling along a different trajectory. The trajectory of each object produces a signature; all objects in one group have related signatures that lie in a manifold.}
                \label{fig:similar-trajectory-signature}
            \end{subfigure}
        \end{tabular}
    \end{center}
    \caption{ 
    Examples and properties of the signatures.}
    \label{figure1}
\end{figure*}

As outlined in the introduction, the right notion of signatures and concepts should allow us to build higher-level concepts from lower-level concepts and allow us to identify common traits in objects by even simply intersecting concepts therefore, in this section, we provide several examples of effectiveness of the signature. 


\subsubsection{Signatures of circle and general concept of concentric circles} We start with learning the concept of a given circle by looking at the null space formed by the points on the circle. Next, by looking at the signatures of several circles and the null-space of the moments of these signatures we learn the general concept of a circle.

\begin{lemma}\label{thm:circlesig}
Consider the manifold $\mathcal M$ given by the unit circle in two dimensions. By using the feature mapping $\phi(x) = [1;x]^{\otimes 2}$, we obtain an invariant signature of a unit circle $T(X) =  ww^\top$ where $w = [-1,0,0,1,0,1]$. This signature can be obtained even from at least three random points from an arc of the circle.
\end{lemma}

\begin{proof}
Note that $\phi(x)$ includes all monomials of degree 2 (that is $\phi(x) =  [1,x_1,x_2,x_1^2,x_1 x_2,x_2^2]$). 
The circle signature is given by $T(X) = ww^\top$
where $w = [-1,0,0,1,0,1]$.
This corresponds to the equation of the circle $\langle T(X), S(x) \rangle = 0$ or $ \langle \phi(x), w \rangle = 0$ or $x_1^2 + x_2^2 - 1 = 0$. This is because points from the arc can only satisfy this one equation of degree $2$. Instead of $\phi(x) = [1;x]^{\otimes 2}$, if we had used a higher power $l$, then too any polynomial equation satisfied by an arc must be satisfied by the circle and must have $x_1^2 + x_2^2 - 1$ as a factor.
\end{proof}

Now, we will provide the result of signature of the general concept of concentric circles.

\begin{lemma}\label{thm:circleconceptsig} 
Individual signatures $T(X)$ of several concentric circles lie in a $1$-d manifold of constant degree; the signature of these signature corresponds to the the concept of a circle (centered at a origin).  
\end{lemma}
\begin{proof}
Note that $T(X)$ entries are polynomials in the radius $r$ of the circle. So there is a generative polynomial in $r$ to produce the signatures of the circles. Rest follows from the result for well behaved manifolds with generative polynomials.
\end{proof}

\subsubsection{Signature of translated and rotated image manifold}
\label{appendix:rotated-translated-signature}

In this section, we show how rotating an image produces the signatures on an analytic manifold which means we can use \cref{thm:analytic-implicit-degree} to obtain a signature of this manifold. Here, we explicitly show the Taylor series that transforms signatures under rotation $\theta$. 

\begin{lemma}
    Let $X$ denote the point cloud of the pixels of image (each pixel can be viewed as concatenation of position and color coordinates) and let $X_{\theta}$ denote the rotation of the image under rotation $\theta$. Then, the signature of $X_\theta$ can be obtained from the signature of $X$ and $\theta$. 
\end{lemma}
\begin{proof}
    Note that each point goes through following transformation:
    \begin{align*}
        x' &= x \cdot cos \theta + y \cdot \sin \theta  \quad \text{and} \quad
        y' = - x \cdot \sin \theta + y \cdot \cos \theta, 
    \end{align*}
    where point $x'$ is obtained by applying the rotation matrix corresponding to $\theta$ angle on point $x$. The second degree moment can be computed as 
    \begin{align*}
        M_{x'} &= \mathbb{E}[x'] = M_{x'} \cos \theta + M_{y'} \sin\theta \\
        M_{y'} &= \mathbb{E}[y'] = -M_x \sin \theta + M_y \cos \theta \\
        M_{x'^2} &= \mathbb{E}[x'^{2}] = M_{x^2} \cos^2 \theta + M_{y^2} \sin^2 \theta \\ 
        & \hspace{2cm} + 2 M_{x} M_y \cos \theta \sin \theta. \\
        M_{y'^2} &= \mathbb{E}[y'^{2}] = M_{x^2} \sin^2 \theta + M_{y^2} \cos^2 \theta \\ 
        & \hspace{2cm} - 2 M_x M_y \cos \theta \sin \theta \\
        M_{x'y'} &= \mathbb{E}[x' y'] = - M_{x^2} \sin \theta \cos \theta + M_{y^2} \sin \theta \cos \theta \\ 
        & \hspace{2cm} + M_x M_y ( \cos^2 \theta - \sin^2 \theta )
    \end{align*}
    For small rotation $\theta$, we can approximate $\sin \theta \approx \theta$ and $\cos \theta \approx 1 - \frac{\theta^2}{2}$ using taylor series with approximation error of $O(\theta^3)$. Using this approximation, we obtain: 
    \begin{align*}
        M_{x'} &= M_x \Big( 1 - \frac{\theta^2}{2} \Big) + M_y \theta \\
        M_{y'} &= - M_x \theta + M_y \Big( 1 - \frac{\theta^2}{2} \Big) \\
        M_{x'^2} &= M_{x^2} (1 - \theta^2) + M_{y^2} \theta^2 + 2 M_{xy} \theta \\
        M_{y'^2} &= M_{x^2} \theta^2 + M_{y^2} (1 - \theta^2 ) - 2 M_{xy} \theta \\
        M_{x'y'} &= - M_{x^2} \theta + M_{y^2} \theta + M_{xy} (1 - 2\theta^2). 
    \end{align*}
    This second order moment allows us to use the power transformation to get the signature of the rotated data manifold. Simplifying above equations, we obtain 
    \begin{align*}
         &\left( \begin{array}{c}
                M_{x'}\\
                M_{y'} \\
                M_{x'^2} \\
                M_{y'^2} \\
                M_{x'y'}
    	\end{array} \right) \approx 
      \left( \begin{array}{c}
                M_{x}\\
                M_{y} \\
                M_{x^2} \\
                M_{y^2} \\
                M_{xy}
    	\end{array} \right)  + \theta \left( \begin{array}{c}
                M_{y}\\
                -M_{x} \\
                M_{xy} \\
                -M_{xy} \\
                M_{y^2} - M_{x^2}
    	\end{array} \right)  + \frac{\theta^2}{2} \left( \begin{array}{c}
                -M_{x}\\
                M_{y} \\
                2M_{y^2} - 2M_{x^2} \\
                2M_{x^2} - 2M_{y^2} \\
                -4 M_{x y}
    	\end{array} \right)
    \end{align*}
    This provides a way to compute the signature $X_\theta$ from the signature of $X$ and $\theta$.
\end{proof}

Next, we describe the manifold obtained by translation of a given image.

\begin{lemma}[Signature of translated image manifold] Let $S(X_{(u, v)})$ be the signature of an image shifted by $(u, v)$. Then, all such signatures $S(X_{(u, v)})$ over different shifts lies on an analytic manifold.
    
\end{lemma}
\begin{proof}
    Suppose each point goes through the following transformation: 
    \begin{align*}
        x' = x + u \quad \text{and} \quad y' = y + v \\
    \end{align*}
    The moments after the translation is given by
    \begin{align*}
        M_{x'} &= M_x + u \\
        M_{y'} &= M_y + v \\
        M_{x'^2} &= M_{x^2} + 2 M_x u + u^2 \\
        M_{y'^2} &= M_{y^2} + 2 M_y v + v^2 \\
        M_{x'y'} &= M_{xy} + M_x v + M_y u + u \cdot v
    \end{align*}
    Suppose the translation $u$ and $v$ are small such that the second order terms are neglible (e.g., $u^2, v^2$ and $u \cdot v$ are negligible), then we can obtain the second order moment statistics as a linear approximation with quadratic approximation error as follows:
    \begin{align*}
    \left( \begin{array}{c}
        M_{x'} \\
        M_{y'} \\
        M_{x'^2} \\
        M_{y'^2} \\
        M_{x'y'} 
    \end{array} \right) \approx
    \left( \begin{array}{c}
     M_x \\
    M_y \\
    M_{x^2} \\
    M_{y^2} \\
    M_{xy} 
    \end{array} \right) 
    + u \left( \begin{array}{c}
         1 \\
         0 \\
         2 \cdot M_x \\
         0 \\
         M_y
    \end{array} \right) 
    + v \left( \begin{array}{c}
         0 \\
         1 \\
         0 \\
         2 \cdot M_y \\
         M_x
    \end{array} \right)
    \end{align*}
\end{proof}

\subsubsection{Signature of an object moving in a specific motion pattern}
\label{appendix-subsec:motionsign}


\begin{lemma}\label{thm:motionsig}
The signatures of the motion of two objects moving with the same velocity function is sufficient to obtain the signature of that velocity function. Any other object with same velocity function will match that signature.
Given a set of points $X_1, X_2$ each moving along separate velocity functions $v_1(t), v_2(t)$, will result in a collection of concept signatures for each point's trajectory. All the concept signature of point trajectories from one set will intersect in the corresponding velocity concept signature. 
\end{lemma}

\begin{proof}
Different objects moving with same velocity can be viewed as subspaces that intersect in the common velocity subspace. For e.g. an object moving in 3d space can be viewed as 1d manifold. And if the velocity is common these are parallel manifolds. By appending a $1$ coordinate these can be seen as 2d manifolds that intersect in the common velocity manifold. By ~\cref{thm:structuresigs} two such signatures are sufficient to obtain the signature $F(V)$ of the velocity manifold. Checking if another object motion $F(O)$ has the same velocity is also easy: check if $F(O) F(V) == F(V)$.

\end{proof}

\section{Connection to transformers}
\label{appendix-new-sec:proofs-connections-transformers}

We showed that keeping track of signature matrices for concepts allows an easy way to manipulate the concepts and also obtain signatures of higher-level concepts. In this section, we provide details about the connection of our learning architecture to the transformer model and proof of \cref{thm:transformer-connection}.

In particular, we show that a two-layer neural network with random weights can obtain $M(X)^2$ upto rotation, i.e., there exists a fixed projection of representation produced by the MLP layers that equals $M(X)^2$. One can use this two-layer neural network as a building block to obtain the higher powers of $M(X)$ and the null-space of a high enough powers of $M(X)$ defines the null-space signature of the concept.  


We denote the data matrix $\matX \in \mathbb{R}^{N \times d}$ where $i^{th}$ row is given by $x_i$.  We consider the square activation and the following two-layer neural network with $m$ hidden units whose weights are initialized using standard Gaussian $\mathcal{N}(0, I)$:
\begin{align*}
    G_{1, j}(x_i) &= \sigma(x_i \cdot r_j), \;\; \;\; \hat{G}_{1, j} = \frac{1}{N} \sum_{i=1}^N G_{1, j}(x_i) \;\; \text{and} \;\;\; G_{2, j} = \sigma( \hat{G}_{1, j}  )
\end{align*}
where $\sigma$ is an element-wise square activation. For wide enough MLP layers, we show that $\hat{G}_{1, j}$ is equivalent to $M(X)$ by a projection and $G_{2, j}$ is equivalent to $M(X)^2$ upto a projection.

\begin{theorem}
\label{thm:transformer-connection}
    There exists a projection of $\hat{G}_{1, j}$ that gives $M(X)$:
    \begin{align*}
        E_{r \sim \mathcal{N}(0, I)}[ r_j^{\otimes 2} \hat{G}_{1, j} ] - (d+1) I = M(X).
    \end{align*}
    Similarly, there exists a projection of the representation generated by first and second layer such that 
    \begin{align*}
        E_{r \sim \mathcal{N}(0, I)} [ r_j^{\otimes 2} \hat{G}_{1, j} + \beta_1 r_j^{\otimes 2} G_{2, j} ] + \beta_2 I = M(X)^2
    \end{align*}
    where $\beta_1$ and $\beta_2$ only depend on $d$.
\end{theorem}




A feedforward network can be thought of transforming the input $x$ to obtain the feature transform $\phi(x)$. In \cref{thm:attentiondistance}, we showed that the cosine similarity of points from the same manifold is high. A similar argument also holds for the feature transform $\phi(x)$ therefore cosine similarity-based attention mechanism will group points from the same concept together and therefore will obtain average across points from the same manifold $\hat{G}_{1, j}$. In \cref{thm:transformer-connection}, we show that the $\hat{G}_{1, j}(X)$ contains the flattened version of the kernel signature $M(X)$. Recall that the null-space signature $T(X)$ can be obtained using repeated multiplication of $(I-M(X))$ therefore, in \cref{thm:transformer-connection}, we prove that the subsequent feedforward network contains $M(X)^2$ matrix. Combining the output of attention and feedforward network, we can obtain $(I - M(X))^2$. Repeating such operations will help in obtaining higher powers of $(I - M(X))$ and hence, the null-space signature $T(X)$. 

\begin{figure*}
    \centering
    \includegraphics[scale=0.6]{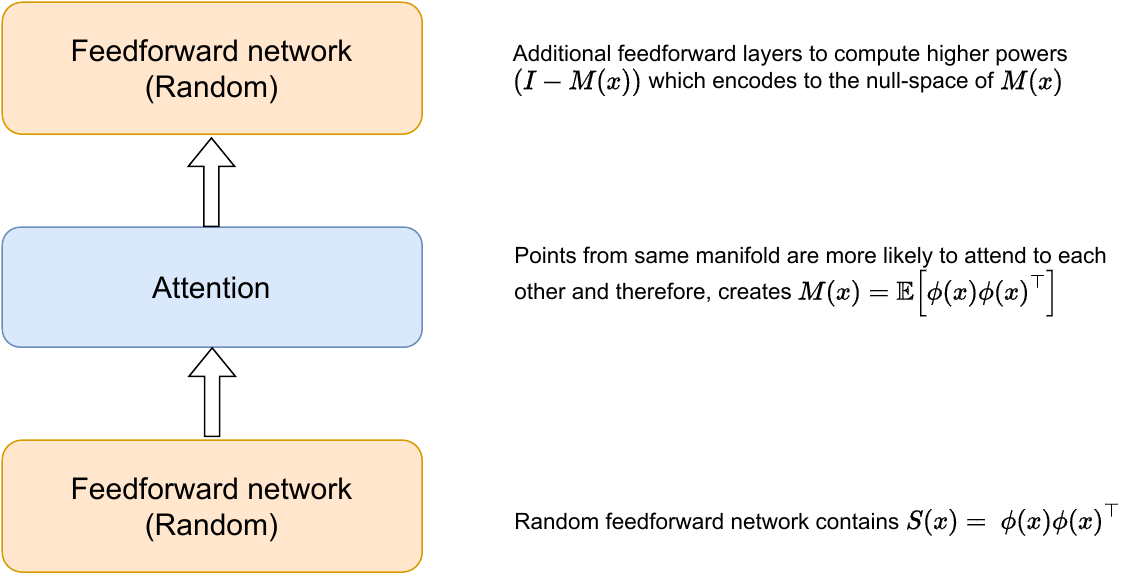}
    \caption{Connection of our learning architecture to the original transformers. A combination of feedforward network and attention module will generate $M(X)$ and a subsequent feedforward layer can generate the nullspace signature $T(X)$ by taking higher power of $I-M(X)$. See \cref{thm:transformer-connection} for more details.}
    \label{fig:transformer-connection}
\end{figure*}

Next, we provide the proof of \cref{thm:transformer-connection}. 

\begin{proof}
    Without loss of generality, we assume that $E[X_i^2] = 1$. By definition, $\hat{G}_{1, j} = \frac{1}{N} \sum_{i=1}^N \sigma( x_i \cdot r_j ) = \frac{1}{N} \sum_{i=1}^N  ( r_j^\top x_i x_i^\top r_j ) = r_j^\top M(X) r_j$.
    The expectation of $\hat{G}_{1, j}$ with respect to random Gaussian weights $r_j$ can be simplified using Stein's lemma. Therefore, we have
    \begin{align*}
        & \mathbb{E}_{r}[ r^{\otimes 2} ( r^\top M(X) r ) ] \\ 
        &= \mathbb{E}_r[ \nabla^{(2)} ( r^\top M(X) r ) + (r^\top M(X) r ) I ] \\ 
        &= M(X) + (d+1)I.
    \end{align*}
    Similarly, simplifying $r_j^{\otimes 2} \hat{G}_{2, j}$ with respect to random Gaussian weights $r_j \sim \mathcal{N}(0, I)$, we have
    \begin{align*}
        & \mathbb{E}[ r^{\otimes 2} \hat{G}_{2, j} ] \\
        &= \mathbb{E}[ \nabla^{(2)} (r^\top M(X) r)^2 + (r^\top M(X) r )^2 ] \\ 
        &=  \mathbb{E}[ 2 (r^\top M(X) r) ( M(X) + I ) + (r^\top M(X) r)^2 I \\ 
        & \hspace{1cm} + (M(X) + I) r r^\top (M(X) + I) ]. \\
        &= 2d (M(X) + I) + M(X)^2 + 2 M(X) + I \\ 
        & \hspace{1cm} + 3\sum_{i=1}^d M_{i,i}(X)^2 I + \sum_{i \neq j} M_{i, i}(X) M_{j, j} (X) I \\
        &= M(X)^2 + (2d + 2) M(X)  + (d^2 + 4d + 1)I \\
    \end{align*}
    Setting $\beta_1 = -(2d + 2)$ and $ \beta_2 = d^2 + 1$, we obtain the result. 
\end{proof}

\end{document}